\documentclass[letter,11pt]{article}

\usepackage{amsmath}
\usepackage{amsthm}
\usepackage{amssymb}
\usepackage{mathtools}
\usepackage{graphicx}
\usepackage{subfigure}
\usepackage[round,comma,authoryear]{natbib}
\usepackage{algorithm}
\usepackage{algorithmic}
\usepackage[hidelinks]{hyperref}
\usepackage{authblk}
\usepackage{fullpage}
\usepackage{multirow}

\usepackage[shortlabels,inline]{enumitem}
\setlist[enumerate,1]{label={\normalfont(\roman*)},leftmargin=1.6cm}

\DeclareMathOperator*{\tr}{tr}
\DeclareMathOperator*{\argmin}{argmin}
\DeclareMathOperator*{\argmax}{argmax}
\DeclareMathOperator*{\diag}{diag}

\newtheorem{theorem}{Theorem}
\newtheorem{corollary}[theorem]{Corollary}

\newtheorem{lemma}[theorem]{Lemma}

\newtheorem{remark}{Remark}

\newcommand{\eqsref}[1]{Eqs.~(\ref{#1})}
\renewcommand{\eqref}[1]{Eq.~(\ref{#1})}
\newcommand{\figref}[1]{Fig.~\ref{#1}}

\title{Safe Screening for Multi-Task Feature Learning\\ with Multiple Data Matrices}
\author[1]{Jie Wang}
\author[1,2]{Jieping Ye}
\affil[1]{Department of Computational Medicine and Bioinformatics, University of Michigan,
            USA}
\affil[2]{Department of Electrical Engineering and Computer Science, University of Michigan,
        	USA}
\date{}

\begin{document}

\maketitle

\begin{abstract}
Multi-task feature learning (MTFL) is a powerful technique in boosting the predictive performance by learning multiple related classification/regression/clustering tasks \mbox{simultaneously}. However, solving the MTFL problem \mbox{remains} challenging when the feature dimension is \mbox{extremely} large. In this paper, we propose a novel screening rule---that is based on the \textbf{d}ual \textbf{p}rojection onto \textbf{c}onvex sets (DPC)---to \mbox{quickly} identify the \emph{inactive features}---that have zero \mbox{coefficients} in the solution vectors across \mbox{all} tasks. One of the appealing features of DPC is that: it is {\it safe} in the sense that the detected inactive features are guaranteed to have \mbox{zero} coefficients in the solution vectors across all tasks. Thus, by removing the inactive features from the training phase, we may have substantial savings in the computational cost and memory usage \emph{without sacrificing accuracy}. To the best of our knowledge, it is the {\it first} screening rule that is applicable to sparse models with {\it multiple} \mbox{data} matrices. A key challenge in deriving DPC is to solve a nonconvex problem. We show that we can solve for the global optimum efficiently via a properly chosen parametrization of the constraint set.  Moreover, DPC has very low \mbox{computational} cost and can be \mbox{integrated}
with any existing solvers. We have evaluated the proposed DPC rule on both synthetic and real data sets. The \mbox{experiments} indicate that \mbox{DPC} is very effective in identifying the inactive features---especially for high dimensional data---which leads to a speedup up to several orders of magnitude.
\end{abstract}


\section{Introduction}
Empirical studies have shown that learning multiple \mbox{related} tasks (MTL) simultaneously often provides superior predictive performance relative to learning each tasks \mbox{independently} \citep{Ando2005,Argyriou2008,Bakker2003,Evgeniou2005,Zhang2006,Chen2013}. This observation also has solid theoretical foundations \citep{Ando2005,Baxter2000,Ben-David2003,Caruana1997}, especially when the training sample size is small for each task. One popular MTL method especially for high-dimensional data is multi-task feature learning (MTFL), which uses the group Lasso penalty to ensure that all tasks select a common set of features \citep{Argyriou2007}. MTFL has found great success in many real-world applications including but not limited to: breast cancer classification \citep{Zhang2010}, disease progression prediction \citep{Zhou2012}, gene data
analysis \citep{Kim2009}, and neural semantic basis discovery \citep{Liu2009}. 
A major issue in MTFL---that is of great practical importance---is to develop efficient solvers \citep{Liu2009a,Sr2012,Wang2013a}. However, it remains challenging to apply the MTFL models to large-scale problems. 

The idea of \emph{screening} has been shown to be very effective in scaling the data and improving the efficiency of many popular sparse models, e.g., Lasso \citep{Ghaoui2012,Wang2013,Wang-JMLR,Xiang2011,Tibshirani2011}, nonnegative Lasso \cite{Wang2014}, group Lasso \citep{Wang2013,Wang-JMLR,Tibshirani2011}, mixed-norm regression \citep{Wang2013a}, $\ell_1$-regularized logistic regression \citep{Wang2014a}, sparse-group Lasso \citep{Wang2014}, support vector machine (SVM) \citep{Ogawa2013,Wang}, and least absolute deviations (LAD) \citep{Wang}. Essentially, screening aims to quickly identify the zero components in the solution vectors such that the corresponding features---called \emph{inactive features} (e.g., Lasso)---or data samples---called \emph{non-support vectors} (e.g., SVM)---can be removed from the optimization. Therefore, the size of the data matrix and the number of variables to be computed can be significantly reduced, which may lead to substantial savings in the computational cost and memory usage \emph{without sacrificing accuracy}. Compared to the solvers without screening, the speedup gained by the screening methods can be several orders of magnitude.

However, we note that all the existing screening \mbox{methods} are only applicable to sparse models with a \emph{single} data matrix.
Therefore, motivated by the challenges posed by large-scale data and the promising performance of existing screening methods, we propose a novel framework for developing effective and efficient screening rules for a popular MTFL model via the \textbf{d}ual \textbf{p}rojection onto \textbf{c}onvex sets (DPC). The framework of DPC extends the state-of-the-art screening rule, called EDPP \citep{Wang-JMLR}, for the \mbox{standard} Lasso problem \citep{Tibshirani1996}---that assumes a \emph{single} data matrix---to a popular MTFL model---that involves \emph{multiple} data matrices across different tasks. To the best of our knowledge, DPC is the \emph{first} screening rule that is applicable to sparse models with \emph{multiple} data matrices. 

The DPC screening rule detects the inactive features by maximizing a convex function over a convex set containing the dual optimal solution, which is a nonconvex problem. To find the region containing the dual optimal solution, we show that the corresponding dual problem can be formulated as a \emph{projection} problem---which admits many desirable geometric properties---by utilizing the \mbox{\emph{bilinearity}} of the \emph{inner product}. Then, by a carefully chosen parameterization of the constraint set, we transform the nonconvex problem to a quadratic programming problem over one quadratic constraint (QP1QC) \citep{Gay1981}, which can be solved for the global optimum efficiently. Experiments on both synthetic and real data sets indicate that the speedup gained by DPC can be orders of magnitude. Moreover, \mbox{DPC} shows better performance as the feature dimension increases, which makes it a very competitive candidate for the applications of very high-dimensional data.


We organize the rest of this paper as follows. In Section \ref{section:Basics}, we briefly review some basics of a popular MTFL model. Then, we derive the dual problem in Section \ref{section:MTFL_dual}. Based on an indepth analysis of the geometric properties of the dual problem and the dual feasible set, we present the proposed DPC screening rule in Section \ref{section:DPC}. In Section \ref{section:experiments}, we evaluate the DPC rule on both synthetic and real data sets. We conclude this paper in Section \ref{section:conclusion}. Please refer to the supplement for proofs not included in the main text.

{\bf Notation}: Denote the $\ell_2$ norm by $\|\cdot\|$. For ${\bf x}\in\mathbb{R}^n$, let its $i^{th}$ component be $x_i$, and the diagonal matrix with the entries of $\mathbf{x}$ on the main diagonal be $\diag(\mathbf{x})$. For a set of positive integers $\{N_t:t=1,\ldots,T,\sum_{t=1}^TN_t=N\}$, we denote the $t^{th}$ subvector of $\mathbf{x}\in\mathbb{R}^N$ by $\mathbf{x}_t$ such that $\mathbf{x}=(\mathbf{x}_1^T,\ldots,\mathbf{x}_T^T)^T$, where $\mathbf{x}_t\in\mathbb{R}^{N_t}$ for $t=1,\ldots,T$. For vectors ${\bf x}, {\bf y}\in\mathbb{R}^n$, we use $\langle{\bf x},\,{\bf y}\rangle$ and ${\bf x}^T{\bf y}$ interchangeably to denote the inner product. For a matrix $M\in\mathbb{R}^{m\times n}$, let ${\bf m}^i$, ${\bf m}_j$, and $m_{ij}$ be its $i^{th}$ row, $j^{th}$ column and $({i,j})^{th}$ entry, respectively. We define the $(2,1)$-norm of $M$ by $\|M\|_{2,1}=\sum_{i=1}^m\|\textbf{m}^i\|$. For two matrices $A,B\in\mathbb{R}^{m\times n}$, we define their inner product by $\langle A, B\rangle=\tr(A^TB)$. 
Let $I$ be the identity matrix. For a convex function $f(\cdot)$, let $\partial f(\cdot)$ be its subdifferential. 
For a vector ${\bf x}$ and a convex set $\mathcal{C}$, the {\it projection operator} is:
\begin{center}
	${\rm P}_{\mathcal{C}}({\bf x}):={\rm argmin}_{{\bf y}\in\mathcal{C}}\hspace{1mm}\tfrac{1}{2}\|{\bf y}-{\bf x}\|.$
\end{center}

\section{Basics}\label{section:Basics}
In this section, we briefly review some basics of a popular MTFL model and mention several equivalent formulations. 

Suppose that we have $T$ learning tasks $\{(X_t,{\bf y}_t):t=1,\ldots,T\}$, where $X_t\in\mathbb{R}^{N_t\times d}$ is the data matrix of the $t^{th}$ task with $N_t$ samples and $d$ features, and ${\bf y}_t\in\mathbb{R}^{N_t}$ is the corresponding response vector. A widely used MTFL model \citep{Argyriou2007} takes the form of 
\begin{align}\label{prob:MTFL}
	\min_{W\in\mathbb{R}^{d\times T}}\hspace{1mm}\sum\nolimits_{t=1}^T\tfrac{1}{2}\|{\bf y}_t-X_t{\bf w}_t\|^2+\lambda\|W\|_{2,1},
\end{align}
where ${\bf w}_t\in\mathbb{R}^d$ is the weight vector of the $t^{th}$ task and $W=({\bf w}_1,\ldots,{\bf w}_T)$. Because the $\|\cdot\|_{2,1}$-norm induces \mbox{sparsity} on the rows of $W$, the weight vectors across all tasks share the same sparse pattern. 
We note that the model in (\ref{prob:MTFL}) is equivalent to several other popular MTFL models. 

The first example introduces a positive weight parameter $\rho_t$ for $t=1,\ldots,T$ to each term in the loss function:
\begin{align*}
	\min_{W\in\mathbb{R}^{d\times T}}\hspace{1mm}\sum\nolimits_{t=1}^T\tfrac{1}{2\rho_t}\|{\bf y}_t-X_t{\bf w}_t\|^2+\lambda\|W\|_{2,1},
\end{align*}
which reduces to (\ref{prob:MTFL}) by setting $\widetilde{\bf y}_t=\frac{{\bf y}_t}{\sqrt{\rho_t}}$ and $\widetilde{X}_t=\frac{{X}_t}{\sqrt{\rho_t}}$. 

The second example introduces another regularizer to (\ref{prob:MTFL}):
\begin{align*}
	\hspace{-3mm}\min_{W\in\mathbb{R}^{d\times T}}\hspace{1mm}\sum\nolimits_{t=1}^T\tfrac{1}{2}\|{\bf y}_t-X_t{\bf w}_t\|^2+\lambda\|W\|_{2,1}+\rho\|W\|_F^2,
\end{align*}
where $\rho$ is a positive parameter and $\|\cdot\|_F$ is the Frobenius norm. Let $I\in\mathbb{R}^{d\times d}$ be the identity matrix and ${\bf 0}$ be the $d$-dimensional vector with all zero entries. By letting
\begin{align*}
	\bar{X}_t=(X_t^T,\sqrt{2\rho}_tI)^T,\hspace{2mm}
	\bar{\bf y}_t=({\bf y}_t^T,{\bf 0}^T)^T,
	t=1,\ldots,T,
\end{align*}
we can also simplify the above MTFL model to (\ref{prob:MTFL}).

In this paper, we focus on developing the DPC screening rule for the MTFL model in (\ref{prob:MTFL}).

\section{The Dual Problem}\label{section:MTFL_dual}

In this section, we show that we can formulate the dual problem of the MTFL model in (\ref{prob:MTFL}) as a projection problem by utilizing the bilinearity of the inner product. 

We first introduce a new set of variables:
\begin{align}\label{eqn:auxiliary_variable}
	{\bf z}_t={\bf y}_t-X_t{\bf w}_t,\,t=1,\ldots,T.
\end{align}
Then, the MTFL model in (\ref{prob:MTFL}) can be written as
\begin{align}
	\min_{W,{\bf z}}\hspace{1mm}&\sum\nolimits_{t=1}^T\tfrac{1}{2}\|{\bf z}_t\|^2+\lambda\|W\|_{2,1},\\ \nonumber
	{\rm s.t.}\hspace{5mm}&{\bf z}_t={\bf y}_t-X_t{\bf w}_t,\,t=1,\ldots,T.
\end{align}
Let $\lambda\theta\in\mathbb{R}^{N}$ be the vector of Lagrangian multipliers. Then, the Lagrangian of (\ref{prob:MTFL}) is
\begin{align}
	L(W,{\bf z};\theta)=&\sum\nolimits_{t=1}^T\tfrac{1}{2}\|{\bf z}_t\|^2+\lambda\|W\|_{2,1}\\ \nonumber
	&+\lambda\sum\nolimits_{t=1}^T\langle\theta_t,{\bf y}_t-X_t{\bf w}_t-{\bf z}_t\rangle.
\end{align}
To get the dual problem, we need to minimize $L(W,{\bf z};\theta)$ over $W$ and ${\bf z}$. We can see that
\begin{align}\label{eqn:MTFL_minf1}
	0=\nabla_{\mathbf{z}}\,L(W,\mathbf{z};\theta)\Rightarrow 
	\argmin\nolimits_{\bf z}\,L(W,\mathbf{z};\theta)=\lambda\theta.
\end{align}
For notational convenience, let
\begin{align*}
	f(W)=\lambda\|W\|_{2,1}-\lambda\sum\nolimits_{t=1}^T\langle\theta_t,X_t{\bf w}_t\rangle.
\end{align*}
Thus, to minimize $L(W,\mathbf{z};\theta)$ with respect to $W$, it is \mbox{equivalent} to minimize $f(W)$, i.e.,
\begin{align*}
	\{\hat{W}:0\in\partial_{W}\,L(\hat{W},\mathbf{z};\theta)\} = \{\hat{W}:0\in\partial\,f(\hat{W})\}.
\end{align*}
By the \emph{bilinearity of the inner product}, we can decouple $f(W)$ into a set of independent subproblems.  
Indeed, we can rewrite the second term of $f(W)$ as
\begin{align}\label{eqn:ipc}
	\sum\nolimits_{t=1}^T\langle\theta_t,X_t{\bf w}_t\rangle=\sum\nolimits_{t=1}^T\langle X_t^T\theta_t,{\bf w}_t\rangle=\langle M,W\rangle,
\end{align}
where $M=(X_1^T\theta_1,\ldots,X_T^T\theta_T)$. \eqref{eqn:ipc} expresses $\langle M, W\rangle$ by the sum of the inner products of the corresponding columns. By the bilinearity of the inner product, we can also express $\langle M,W\rangle$ by the sum of the inner products of the corresponding rows:
\begin{align}\label{eqn:ipr}
	\sum\nolimits_{t=1}^T\langle\theta_t,X_t{\bf w}_t\rangle=\langle M,W\rangle=\sum\nolimits_{\ell=1}^d\langle {\bf m}^{\ell},{\bf w}^{\ell}\rangle.
\end{align}
Denote the $j^{th}$ column of $X_t$ by ${\bf x}_j^{(t)}$. We can see that
\begin{align}\label{eqn:Mr}
	\hspace{-3mm}{\bf m}^{\ell}=(\langle{\bf x}_{\ell}^{(1)},\theta_1\rangle,\langle{\bf x}_{\ell}^{(2)},\theta_2\rangle,\ldots,\langle{\bf x}_{\ell}^{(T)},\theta_T\rangle).
\end{align}
Moreover, as $\|W\|_{2,1}=\sum\nolimits_{\ell=1}^d\|{\bf w}^{\ell}\|$,  
\eqsref{eqn:ipr} implies that:
\begin{align*}
	f(W)=\lambda\sum\nolimits_{\ell=1}^df^{(\ell)}({\bf w}^{\ell}),
\end{align*}
where $f^{(\ell)}({\bf w}^{\ell})=\|{\bf w}^{\ell}\|-\langle{\bf m}^{\ell},{\bf w}^{\ell}\rangle$.
Thus, to minimize $f(W)$, we can minimize each $f^{(\ell)}({\bf w}^{\ell})$ \mbox{separately}. The subdifferential counterpart of the Fermat's rule \citep{Bauschke2011}, i.e., $0\in\partial f^{(\ell)}(\hat{\bf w}^{\ell})$, \mbox{yields}:
\begin{align}\label{eqn:MTFL_minf2_Fermat}
	{\bf m}^{\ell}\in
	\begin{dcases}
		{\hat{\bf w}^{\ell}}/{\|\hat{\bf w}^{\ell}\|},\hspace{19mm}\mbox{if}\,\hat{\bf w}^{\ell}\neq0,\\
		\{{\bf u}\in\mathbb{R}^{d}:\,\|{\bf u}\|\leq1\},\hspace{3mm}\mbox{if}\,\hat{\bf w}^{\ell}=0,
	\end{dcases}
\end{align}
where $\hat{\bf w}^{\ell}$ is the minimizer of $f^{(\ell)}(\cdot)$.

We note that \eqref{eqn:MTFL_minf2_Fermat} implies $\|{\bf m}^{\ell}\|\leq1$. If this is not the case, then $f^{\ell}(\cdot)$ is not lower bounded (see the supplements for discussions), i.e., $\min_{\mathbf{w}^{\ell}}\,f^{\ell}(\mathbf{w}^{\ell})=-\infty$.  
Thus, by \eqsref{eqn:MTFL_minf1} and (\ref{eqn:MTFL_minf2_Fermat}), the dual function is
\begin{align}\label{eqn:MTFL_dual}
	\hspace{-2mm}q(\theta)=&\min\nolimits_{W,{\bf z}}\hspace{1mm}L(W,{\bf z};\theta)\\ \nonumber
	=&
	\begin{dcases}
		-\tfrac{\lambda^2}{2}\|\theta\|^2+\lambda\langle\theta,{\bf y}\rangle,\hspace{1mm}\|{\bf m}^{\ell}\|\leq1,\,\forall\,\ell\in\{1,\ldots,d\},\\
		-\infty,\hspace{24mm}\mbox{otherwise}.
	\end{dcases}
\end{align}
Maximizing $q(\theta)$ yields
the dual problem of (\ref{prob:MTFL}) as follows:
\begin{align}\label{prob:MTFL_dual_max}
	\max_{\theta}\hspace{2mm}&\tfrac{1}{2}\|{\bf y}\|^2-\tfrac{\lambda^2}{2}\left\|\tfrac{{\bf y}}{\lambda}-\theta\right\|^2,\\ \nonumber
	\mbox{s.t.}\,&\sum\nolimits_{t=1}^T\langle{\bf x}_{\ell}^{(t)},\theta_t\rangle^2\leq1,\,\ell=1,\ldots,d.
\end{align}
It is evident that the problem in (\ref{prob:MTFL_dual_max}) is equivalent to
\begin{align}\label{prob:MTFL_dual}
	\min_{\theta}\hspace{2mm}&\tfrac{1}{2}\left\|\tfrac{{\bf y}}{\lambda}-\theta\right\|^2,\\ \nonumber
	\mbox{s.t.}\,&\sum\nolimits_{t=1}^T\langle{\bf x}_{\ell}^{(t)},\theta_t\rangle^2\leq1,\,\ell=1,\ldots,d.
\end{align}
In view of (\ref{prob:MTFL_dual}), it is indeed a projection problem. Let $\mathcal{F}$ be the feasible set of (\ref{prob:MTFL_dual}). Then, the optimal solution of (\ref{prob:MTFL_dual}), denoted by $\theta^*(\lambda)$, is the projection of ${\bf y}/\lambda$ onto $\mathcal{F}$, namely,
\begin{align}\label{eqn:theta*_proj}
	\theta^*(\lambda)={\rm P}_{\mathcal{F}}\left(\tfrac{{\bf y}}{\lambda}\right).
\end{align}

\section{The DPC Rule}\label{section:DPC}

In this section, we present the proposed DPC screening rule for the MTFL model in (\ref{prob:MTFL}). Inspired by the Karush-Kuhn-Tucker (KKT) conditions \citep{Guler2010}, in Section \ref{subsection:MTFL_guideline_DPC}, we first present the general guidelines. The most challenging part lies in two folds: 1) we need to estimate the dual optimal solution as accurately as possible; 2) we need to solve a nonconvex optimization problem. In Section \ref{subsection:estimation}, we give an accurate estimation of the dual optimal solution based on the geometric properties of the projection operators. Then, in Section \ref{subsection:nonconvex}, we show that we can efficiently solve for the global optimum to the nonconvex problem.  We present the DPC rule for the MTFL model (\ref{prob:MTFL}) in Section \ref{subsection:DPC}.

\subsection{Guidelines for Developing DPC}\label{subsection:MTFL_guideline_DPC}

We present the general guidelines to develop screening rules for the MTFL model (\ref{prob:MTFL}) via the KKT conditions. 

Let $W^*(\lambda)=({\bf w}_1^*(\lambda),\ldots,{\bf w}_T^*(\lambda))$ be the optimal solution (\ref{prob:MTFL}). By \eqsref{eqn:auxiliary_variable}, (\ref{eqn:MTFL_minf1}) and (\ref{eqn:MTFL_minf2_Fermat}), the KKT conditions are:
\begin{align}\label{eqn:MTFL_KKT1}
	&{\bf y}_t=X_t{\bf w}_t^*(\lambda)+\lambda\theta_t^*(\lambda),\,t=1,\ldots,T,\\
	\label{eqn:MTFL_KKT2}
	&\hspace{-4mm}g_{\ell}(\theta^*(\lambda))\in
	\begin{dcases}
		1,\hspace{8.5mm}{\rm if}\,({\bf w}^{\ell})^*(\lambda)\neq0,\\
		[-1,1],\hspace{1mm}{\rm if}\,({\bf w}^{\ell})^*(\lambda)=0,
	\end{dcases}\ell=1,\ldots,d.
\end{align}
where $({\bf w}^{\ell})^*(\lambda)$ is the $\ell^{th}$ row of $W^*(\lambda)$, and
\begin{align}\label{eqn:MTFL_dual_constraint}
	g_{\ell}(\theta)=\sum\nolimits_{t=1}^T\langle{\bf x}_{\ell}^{(t)},\theta_t\rangle^2,\,\ell=1,\ldots,d.
\end{align} 
For $\ell=1,\ldots,d$, \eqref{eqn:MTFL_KKT2} yields
\begin{align}\tag{\bf R}\label{rule}
	g_{\ell}(\theta^*(\lambda))<1\Rightarrow ({\bf w}^{\ell})^*(\lambda)=0.
\end{align}
The rule in (\ref{rule}) provides a method to identify the rows in $W^*(\lambda)$ that  have \emph{only} zero entries. However, (\ref{rule}) is not applicable to real applications, as it assumes knowledge of $\theta^*(\lambda)$, and solving the dual problem (\ref{prob:MTFL_dual}) could be as expensive as solving the primal problem (\ref{prob:MTFL}). Inspired by SAFE \citep{Ghaoui2012}, we can first estimate a set ${\bf \Theta}$ that contains $\theta^*(\lambda)$, and then relax (\ref{rule}) as follows:
\begin{align}\tag{\bf R$^*$}\label{rule*}
	\hspace{-2mm}\max\nolimits_{\theta\in{\bf \Theta}}\,g_{\ell}(\theta)<1\Rightarrow ({\bf w}^{\ell})^*(\lambda)=0,\,\ell=1,\ldots,d.
\end{align}
Therefore, to develop a screening rule for the MTFL model in (\ref{prob:MTFL}), (\ref{rule*}) implies that: 1) we need to estimate a region ${\bf \Theta}$---that turns out to be a ball (please refer to Section \ref{subsection:estimation})---containing $\theta^*(\lambda)$; 2) we need to solve the maximization problem---that \mbox{turns} out to be nonconvex (please refer to Section \ref{subsection:nonconvex})---on the left hand side of (\ref{rule*}). 


\subsection{Estimation of the Dual Optimal Solution}\label{subsection:estimation}

Based on the geometric properties of the dual problem (\ref{prob:MTFL_dual}) that is a projection problem, we first derive the closed form solutions of the primal and dual problems for \mbox{specific} values of $\lambda$ in Section \ref{subsubsection:MTFL_closed_form_solution}, and then give an accurate estimation of $\theta^*(\lambda)$ for the general cases in Section \ref{subsubsection:MTFL_estimation_general}.

\subsubsection{Closed form solutions}\label{subsubsection:MTFL_closed_form_solution}
The primal and dual optimal solutions $W^*(\lambda)$ and $\theta^*(\lambda)$ are generally unknown. However, when the value of $\lambda$ is sufficiently large, we expect that $W^*(\lambda)=0$, and $\theta^*(\lambda)=\frac{\bf y}{\lambda}$ by \eqref{eqn:MTFL_KKT1}. The following theorem confirms this.
\begin{theorem}\label{thm:MTFL_primal_dual_closed_form}
	For the MTFL model in (\ref{prob:MTFL}), let 
	\begin{align}\label{eqn:MTFL_lambdamx}
		\lambda_{\rm max}=\max_{\ell=1,\ldots,d}\,\sqrt{\sum\nolimits_{t=1}^T\langle{\bf x}_{\ell}^{(t)},{\bf y}\rangle^2}.
	\end{align}
	Then, the following statements are equivalent:
	
	\hspace{5mm}$\frac{\bf y}{\lambda}\in\mathcal{F}\Leftrightarrow\theta^*(\lambda)=\frac{\bf y}{\lambda}\Leftrightarrow W^*(\lambda)=0\Leftrightarrow\lambda\geq\lambda_{\rm max}$.
\end{theorem}

\begin{remark}\label{remark:MTFL_lambda_range}
	Theorem \textup{\ref{thm:MTFL_primal_dual_closed_form}} indicates that: both the primal and \mbox{dual} optimal solutions of the MTFL model \textup{(\ref{prob:MTFL})} admit closed form solutions for $\lambda\geq\lambda_{\rm max}$. Thus, we will focus on the cases with $\lambda\in(0,\lambda_{\rm max})$ in the rest of this paper.
\end{remark}


\subsubsection{The general cases}\label{subsubsection:MTFL_estimation_general}

Theorem \ref{thm:MTFL_primal_dual_closed_form} gives a closed form solution of $\theta^*(\lambda)$ for $\lambda\geq\lambda_{\rm max}$. Therefore, we can estimate $\theta^*(\lambda)$ with $\lambda<\lambda_{\rm max}$ in terms of a known $\theta^*(\lambda_0)$. Specifically, we can simply set $\lambda_0=\lambda_{\rm max}$ and utilize the result $\theta^*(\lambda_{\rm max})={\bf y}/\lambda_{\rm max}$. 
To make this paper self-contained, we first review some geometric properties of projection operators.
\begin{theorem}\label{thm:normal_cone}
	\textup{\citep{Ruszczynski2006}}
	Let $\mathcal{C}$ be a nonempty closed convex set. Then, for any point $\bar{\mathbf{u}}$, we have
	\begin{align*}
		\mathbf{u}=\mathbf{P}_{\mathcal{C}}(\overline{\mathbf{u}})\Leftrightarrow \overline{\mathbf{u}}-\mathbf{u}\in N_{\mathcal{C}}(\mathbf{u}),
	\end{align*}
	where $N_{\mathcal{C}}(\mathbf{u})=\{\mathbf{v}:\langle\mathbf{v},\mathbf{u}'-\mathbf{u}\rangle\leq0,\,\forall \mathbf{u}'\in\mathcal{C}\}$ is called
	the normal cone to $\mathcal{C}$ at $\mathbf{u}\in\mathcal{C}$.
\end{theorem}
Another useful property of the projection operator in estimating $\theta^*(\lambda)$ is the so-called \emph{firmly nonexpansiveness}.
\begin{theorem}\label{thm:firmly_nonexpansive}
	\textup{\citep{Bauschke2011}} Let $\mathcal{C}$ be a nonempty closed convex subset of a Hilbert space $\mathcal{H}$. The projection operator with respect to $\mathcal{C}$ is firmly nonexpansive, namely, for any ${\bf u}_1, {\bf u}_2\in\mathcal{H}$, 
	\begin{align}\label{ineqn:nonexpansive}\nonumber
		&\|\textup{P}_{\mathcal{C}}({\bf u}_1)-\textup{P}_{\mathcal{C}}({\bf u}_2)\|^2+\|(I-\textup{P}_{\mathcal{C}})({\bf u}_1)-(I-\textup{P}_{\mathcal{C}})({\bf u}_2)\|^2\\ 
		&\hspace{50mm}\leq\|{\bf u}_1-{\bf u}_2\|^2.
	\end{align}
\end{theorem}
The firmly nonexpansiveness of projection operators leads to the following useful result.
\begin{corollary}\label{corollary:convex_set_projection}
	Let $\mathcal{C}$ be a nonempty closed convex subset of a Hilbert space $\mathcal{H}$ and $0\in\mathcal{C}$. For any ${\bf u}\in\mathcal{H}$, we have:
	
	\hspace{5mm}\textup{1.} $\|\textup{P}_{\mathcal{C}}({\bf u})\|^2+\|{\bf u}-\textup{P}_{\mathcal{C}}({\bf u})\|^2\leq\|{\bf u}\|^2$.
	
	\hspace{5mm}\textup{2.} $\langle{\bf u},{\bf u}-\textup{P}_{\mathcal{C}}({\bf u})\rangle\geq0$.
\end{corollary}
\begin{remark}
	Part \textup{1} of \textup{Corollary \ref{corollary:convex_set_projection}} indicates that: if a closed convex set $\mathcal{C}$ contains the origin, then, for any point ${\bf u}$, the norm of its projection with respect to $\mathcal{C}$ is upper bounded by the norm of $\|{\bf u}\|$. The second part is a useful consequence of the first part and plays a crucial role in the estimation of the dual optimal solution (see Theorem \textup{\ref{thm:MTFL_estimation}}). 
\end{remark}
We are now ready to present an accurate estimation of the dual optimal solution $\theta^*(\lambda)$. 
\begin{theorem}\label{thm:MTFL_estimation}
	For the MTFL model in \textup{(\ref{prob:MTFL})}, suppose that $\theta^*(\lambda_0)$ is known with $\lambda_0\in(0,\lambda_{\rm max}]$. Let $g_{\ell}$ be given by \textup{\eqref{eqn:MTFL_dual_constraint}} for $\ell=1,\ldots,d$, and	
	\begin{align}
		\ell_{*}\in\left\{\argmax\nolimits_{\ell=1,\ldots,d}\,g_{\ell}({\bf y})\right\}.
	\end{align}
	For any $\lambda\in(0,\lambda_0)$,  we define 
	\begin{align}\label{eqn:MTFL_n}
		\textbf{\textup{n}}(\lambda_0)&=
		\begin{dcases}
			\tfrac{{\bf y}}{\lambda_0}-\theta^*(\lambda_0),\hspace{3mm}{\rm if}\hspace{2mm}\lambda_0\in(0,\lambda_{\rm max}),\\
			\nabla g_{\ell_*}\left(\tfrac{\mathbf{y}}{\lambda_{\rm max}}\right),\hspace{1.5mm}{\rm if}\hspace{2mm}\lambda_0=\lambda_{\rm max}.
		\end{dcases}\\
		\label{eqn:MTFL_r}
		{\bf r}(\lambda,\lambda_0)&=\tfrac{{\bf y}}{\lambda}-\theta^*(\lambda_0),\\ \label{eqn:MTFL_r_perp}
		{\bf r}^{\perp}(\lambda,\lambda_0)&={\bf r}(\lambda,\lambda_0)-\frac{\langle\textbf{\textup{n}}(\lambda_0),{\bf r}(\lambda,\lambda_0)\rangle}{\|\textbf{\textup{n}}(\lambda_0)\|^2}\textbf{\textup{n}}(\lambda_0).
	\end{align}
	Then, the following holds:
	
	\hspace{1mm}\textup{1.} $\textbf{\textup{n}}(\lambda)\in N_{\mathcal{F}}(\theta^*(\lambda))$, 
	
	\hspace{1mm}\textup{2.} $\langle \mathbf{y},\mathbf{n}(\lambda_0)\rangle\geq0$, 
	
	\hspace{1mm}\textup{3.} $\langle \mathbf{r}(\lambda,\lambda_0),\mathbf{n}(\lambda_0)\rangle\geq0$, 
	
	\hspace{1mm}\textup{4.} $\left\|\theta^*(\lambda)-\left(\theta^*(\lambda_0)+\frac{1}{2}{\bf r}^{\perp}(\lambda,\lambda_0)\right)\right\|\leq\frac{1}{2}\|{\bf r}^{\perp}(\lambda,\lambda_0)\|$. 
\end{theorem}
%
Consider Theorem \ref{thm:MTFL_estimation}. Part 1 characterizes $\theta^*(\lambda)$ via the normal cone. Parts 2 and 3 illustrate key geometric identities that lead to the accurate estimation of $\theta^*(\lambda)$ in part 4 (see supplement for details).
\begin{remark}
	The estimation of the dual optimal solution in DPC and EDPP \textup{\citep{Wang-JMLR}}---that is for Lasso---are both based on the geometric properties of the projection operators. Thus, the formulas of the estimation in Theorem \textup{\ref{thm:MTFL_estimation}} are similar to that of \mbox{EDPP}. However, we note that the estimations in DPC and EDPP are determined by the completely different geometric structures of the corresponding dual feasible sets. Problem \textup{(\ref{prob:MTFL_dual})} implies that the dual feasible set of the MTFL model \textup{(\ref{prob:MTFL})} is much more \mbox{complicated} than that of Lasso---which is a polytope \textup{(}the intersection of a set of closed half spaces\textup{)}. Therefore, the estimation of the dual optimal solution in DPC is much more challenging than that of EDPP, e.g., we need to find a vector in the normal cone to the dual feasible set at $\mathbf{y}/\lambda_{\rm max}$ \textup{[}see $\mathbf{n}(\lambda_{\rm max})$\textup{]}.
\end{remark}
For notational convenience, let 
\begin{align}
	\textbf{o}(\lambda,\lambda_0)=\theta^*(\lambda_0)+\frac{1}{2}{\bf r}^{\perp}(\lambda,\lambda_0).
\end{align}
Theorem \ref{thm:MTFL_estimation} implies that $\theta^*(\lambda)$ lies in the ball:
\begin{align}\label{eqn:MTFL_estimation_ball}
	\hspace{-3mm}{\bf \Theta}(\lambda,\lambda_0)=\left\{\theta:\left\|\theta-\textbf{o}(\lambda,\lambda_0)\right\|\leq\frac{1}{2}\|{\bf r}^{\perp}(\lambda,\lambda_0)\|\right\}.
\end{align}

\subsection{Solving the Nonconvex Problem}\label{subsection:nonconvex}

In this section, we solve the optimization problem in (\ref{rule*}) with ${\bf \Theta}$ given by ${\bf \Theta}(\lambda,\lambda_0)$ [see \eqref{eqn:MTFL_estimation_ball}], namely,
\begin{align}\label{prob:nonconvex_ball}
	\hspace{-3mm}s_{\ell}(\lambda,\lambda_0)=\max_{\theta\in{\bf \Theta}(\lambda,\lambda_0)}\,\left\{g_{\ell}(\theta)=\sum\nolimits_{t=1}^T\langle{\bf x}_{\ell}^{(t)},\theta_t\rangle^2\right\}.
\end{align}
Although $g_{\ell}(\cdot)$ and ${\bf \Theta}(\lambda,\lambda_0)$ are convex, problem (\ref{prob:nonconvex_ball}) is nonconvex, as it is a maximization problem. However, we can efficiently solve for the \emph{global} optimal solutions to (\ref{prob:nonconvex_ball}) by transforming it to a QP1PC via a parametrization of the constraint set. We first cite the following result.
\begin{theorem}\label{thm:QP1PC}
	\textup{\citep{Gay1981}} Let $H$ be a symmetric matrix and $D$ be a positive definite matrix. Consider 
	\begin{align}\label{prob:QP1PC1}
		\min_{\|D\mathbf{u}\|\leq\Delta}\,\psi(\mathbf{u})=\frac{1}{2}\mathbf{u}^TH\mathbf{u}+\mathbf{q}^T\mathbf{u},
	\end{align}
	where $\Delta>0$. Then, $\mathbf{u}^*$ minimizes $\psi(\mathbf{u})$ over the constraint set if and only if
	there exists $\alpha^*\geq0$---that is unique---such that $(H+\alpha^*D^TD)\mathbf{u}^*$ is positive semidefinite,
	\begin{align}\label{eqn:QP1PC1}
		&(H+\alpha^*D^TD)\mathbf{u}^*=-\mathbf{q},
		\\\label{eqn:QP1PC2}
		&\|D\mathbf{u}^*\|=\Delta,\,\mbox{if}\,\alpha^*>0.
	\end{align}
\end{theorem}
We are now ready to solve for $s_{\ell}(\lambda,\lambda_0)$.
\begin{theorem}\label{thm:nonconvex}
	Let $\mathbf{o}=\mathbf{o}(\lambda,\lambda_0)$ and $\mathbf{u}^*$ be the optimal solution of problem \textup{(\ref{prob:QP1PC1})} with 
	$\Delta=\frac{1}{2}\|\mathbf{r}^{\perp}(\lambda,\lambda_0)\|$, $D=I$,
	\begin{align*}
		H=&-\diag(2\|\mathbf{x}\|_{\ell}^{(1)},\ldots,2\|\mathbf{x}\|_{\ell}^{(T)}),\\ 
		\mathbf{q}=&-\left(2\|\mathbf{x}_{\ell}^{(1)}\||\langle\mathbf{x}_{\ell}^{(1)},\mathbf{o}_1\rangle|,\ldots,2\|\mathbf{x}_{\ell}^{(T)}\||\langle\mathbf{x}_{\ell}^{(T)},\mathbf{o}_T\rangle|\right)^T,
	\end{align*} 
	namely, there exists a $\alpha^*\geq0$ such that $\alpha^*$ and $\mathbf{u}^*$ solve \textup{\eqsref{eqn:QP1PC1}} and \textup{(\ref{eqn:QP1PC2})}. Let
	\begin{align*}
		\rho_{\ell}=\max_{t=1,\ldots,T}\,\|{\bf x}_{\ell}^{(t)}\|,\hspace{2mm}\mathcal{I}_{\ell}=\left\{t_*:\,\|{\bf x}_{\ell}^{(t_*)}\|=\rho_{\ell}\right\}.
	\end{align*}
	Then, the following hold:	
	
	\textup{1}. $\alpha^*$ is unique, and $\alpha^*\geq2\rho_{\ell}$.
	
	\textup{2}. We define $\bar{\mathbf{u}}\in\mathbb{R}^{T}$ by
	\begin{align*}
		\bar{u}_t=
		\begin{dcases}
			-{q_t}/({h_{tt}+2\rho_{\ell}}),\hspace{2mm}\mbox{if}\,\,t\notin\mathcal{I}_{\ell},\\
			0,\hspace{24mm}\mbox{otherwise}.
		\end{dcases}
	\end{align*}
	Then, we have
	\begin{align*}
		\hspace{-3mm}\alpha^*\in
		\begin{dcases}
			2\rho_{\ell},\hspace{2mm}\mbox{if}\,\|\bar{\mathbf{u}}\|\leq\Delta,\,\mbox{and}\,\,\langle\mathbf{x}_{\ell}^{(t_*)},\mathbf{o}_{t_*}\rangle=0,\,\mbox{for}\,\,t_*\in\mathcal{I}_{\ell},\\
			(2\rho_{\ell},\infty),\hspace{2mm}\mbox{otherwise}.
		\end{dcases}
	\end{align*}
	
	\textup{3}. Let $\mathcal{V}=\{\mathbf{v}\in\mathbb{R}^{T}:v_t=0\,\mbox{for}\,t\notin\mathcal{I}_{\ell},\|\bar{\mathbf{u}}+\mathbf{v}\|=\Delta\}$. Then, we have
	\begin{align*}
		\mathbf{u}^*\in
		\begin{dcases}
			\bar{\mathbf{u}}+\mathbf{v},\,\mathbf{v}\in\mathcal{V},\hspace{7mm}\mbox{if}\,\,\alpha^*=2\rho_{\ell},\\
			-(H+\alpha^*I)^{-1}\mathbf{q},\hspace{2mm}\mbox{otherwise}.
		\end{dcases}
	\end{align*}
	
	\textup{4}. The maximum value of problem \textup{(\ref{prob:nonconvex_ball})} is given by
	\begin{align*}
		s_{\ell}(\lambda,\lambda_0)=\sum\nolimits_{t=1}^T\langle\mathbf{x}_{\ell}^{(t)},\mathbf{o}_t\rangle^2+\frac{\alpha^*}{2}\Delta^2-\frac{1}{2}\mathbf{q}^T\mathbf{u}^*.
	\end{align*}
\end{theorem}
\begin{proof}
	We first transform problem (\ref{prob:nonconvex_ball}) to a QP1PC by a 
	parameterization of $\mathbf{\Theta}(\lambda,\lambda_0)$:
	\begin{align*}
		\hspace{-2mm}&\mathbf{\Theta}(\lambda,\lambda_0)\\= \nonumber
		&\left\{
		\begin{pmatrix}
			\mathbf{o}_1+u_1{\theta}_1\\
			\vdots\\
			\mathbf{o}_T+u_T{\theta}_T
		\end{pmatrix}
		:\|\mathbf{u}\|\leq r,\|\theta_t\|\leq1,,t=1,\ldots,T\right\},
	\end{align*}
	where $\mathbf{u}=(u_1,\ldots,u_T)^T$.
	We define
	\begin{align*}
		h_{\ell}(\mathbf{u},\theta)=g_{\ell}\left(\begin{pmatrix}
			\mathbf{o}_1+u_1{\theta}_1\\
			\vdots\\
			\mathbf{o}_T+u_T{\theta}_T
		\end{pmatrix}\right).
	\end{align*}  
	Thus, problem (\ref{prob:nonconvex_ball}) becomes
	\begin{align*}
		s_{\ell}(\lambda,\lambda_0)=\max_{\|\mathbf{u}\|\leq \Delta}\,\left\{\max_{\{\theta:\|\theta_t\|\leq1,t=1,\ldots,T\}}\,h_{\ell}(\mathbf{u},\theta)\right\}.
	\end{align*}
	By the Cauchy-Schwartz inequality, for a fixed $\mathbf{u}$, we have
	\begin{align*}
		&\phi(\mathbf{u})=\max_{\{\theta:\|\theta_t\|\leq1,t=1,\ldots,T\}}\,h_{\ell}(\mathbf{u},\theta)\\ \nonumber
		&=\sum\nolimits_{t=1}^Tu_t^2\|\mathbf{x}_{\ell}^{(t)}\|^2+2|u_t|\|\mathbf{x}_{\ell}^{(t)}\||\langle\mathbf{x}_{\ell}^{(t)},\mathbf{o}_t\rangle|+\langle\mathbf{x}_{\ell}^{(t)},\mathbf{o}_t\rangle^2.
	\end{align*} 
	Let $-\psi(\mathbf{u})=\sum\nolimits_{t=1}^Tu_t^2\|\mathbf{x}_{\ell}^{(t)}\|^2+2u_t\|\mathbf{x}_{\ell}^{(t)}\||\langle\mathbf{x}_{\ell}^{(t)},\mathbf{o}_t\rangle|$.
	We can see that
	\begin{align*}
		\max\nolimits_{\|\mathbf{u}\|\leq r}\,\phi(\mathbf{u})
		=\max\nolimits_{\|\mathbf{u}\|\leq r}\,-\psi(\mathbf{u})+\sum\nolimits_{t=1}^T\langle\mathbf{x}_{\ell}^{(t)},\mathbf{o}_t\rangle^2.
	\end{align*}
	Thus, problem (\ref{prob:nonconvex_ball}) becomes
	\begin{align*}
		s_{\ell}(\lambda,\lambda_0)=-\min\nolimits_{\|\mathbf{u}\|\leq r}\,\psi(\mathbf{u})+\sum\nolimits_{t=1}^T\langle\mathbf{x}_{\ell}^{(t)},\mathbf{o}_t\rangle^2.
	\end{align*}
	Therefore, to solve (\ref{prob:nonconvex_ball}), it suffices to solve problem (\ref{prob:QP1PC1}) with $\Delta$, $D$, $H$, and $\mathbf{q}$ as in the theorem.
	
	The statement follows immediately from Theorem \ref{thm:QP1PC}.
\end{proof}
\begin{remark}
	To develop the DPC rule, (\ref{rule*}) implies that we only need the maximum value of problem \textup{(\ref{prob:nonconvex_ball})}. Thus, Theorem \ref{thm:QP1PC} does not show the global optimal solutions. However, in view of the proof, we can easily compute the global optimal solutions in terms of $\alpha^*$ and $\mathbf{u}^*$.
\end{remark}
\textbf{Computing $\alpha^*$ and $\mathbf{u}^*$}
Consider Theorem \textup{\ref{thm:nonconvex}}. If $\|\bar{\mathbf{u}}\|\leq\Delta$ and $\langle\mathbf{x}_{\ell}^{(t_*)},\mathbf{o}_{t_*}\rangle=0$ for $t_*\in\mathcal{I}_{\ell}$, then $\alpha^*$ and $\mathbf{u}^*$ admit closed form solutions. Otherwise, $\alpha^*$ is strictly larger than $2\rho_{\ell}$, which implies that $H+\alpha^*I$ is positive definite and invertible. If this is the case, we apply Newton's method \textup{\citep{Gay1981}} to find $\alpha^*$ as follows. Let 
\begin{align*}
	\varphi(\alpha)=\|(H+\alpha I)^{-1}\mathbf{q}\|^{-1}-\Delta^{-1}.
\end{align*}
Because $\varphi(\cdot)$ is strictly increasing on $(2\rho_{\ell},\infty)$, $\alpha^*$ is the \mbox{unique} root of $\varphi(\cdot)$ on $(2\rho_{\ell},\infty)$.	
Let $\alpha_0=2\rho_{\ell}$. Then, the $k^{th}$ iteration of Newton's method to solve $\varphi(\alpha^*)=0$ is:
\begin{align}
	\mathbf{u}_k=&-(H+\alpha_{k-1} I)^{-1}\mathbf{q},\\
	\alpha_k = & \alpha_{k-1}+\|\mathbf{u}_k\|^2\frac{\|\mathbf{u}_k\|-\Delta}{\Delta\mathbf{u}_{k}^T(H+\alpha_{k-1}I)^{-1}\mathbf{u}_k}.
\end{align}
As pointed out by \textup{\citet{More1983}}, Newton's method is very efficient to find $\alpha^*$ as $\varphi(\alpha)$ is almost linear on $(2\rho_{\ell},\infty)$. Our experiments indicates that five iterations usually leads to an accuracy higher than $10^{-15}$.

\subsection{The Proposed DPC Rule}\label{subsection:DPC}

As implied by \ref{rule*}, we present the proposed screening rule, DPC, for the MTFL model (\ref{prob:MTFL}) in the following theorem.
\begin{theorem}\label{thm:DPC}
	For the MTFL model \textup{(\ref{prob:MTFL})}, suppose that $\theta^*(\lambda_0)$ is known with $\lambda_0\in(0,\lambda_{\rm max}]$. Then, we have
	\begin{align*}
		s_{\ell}(\lambda,\lambda_0)<1\Rightarrow(\mathbf{w}^{\ell})^*(\lambda)=0,\,\lambda\in(0,\lambda_0),
	\end{align*}
	where $s_{\ell}(\lambda,\lambda_0)$ is given by Theorem \textup{\ref{thm:nonconvex}}.
\end{theorem} 
In real applications, the optimal parameter value of $\lambda$ is generally unknown. Commonly used approaches to determine an appropriate value of $\lambda$, such as cross validation and stability selection, need to solve the MTFL model over a grid of tuning parameter values $\lambda_1>\lambda_2>\ldots>\lambda_{\mathcal{K}}$, which is very time consuming. Inspired by the ideas of Strong Rule \citep{Tibshirani2011} and SAFE \citep{Ghaoui2012}, we develop the sequential version of DPC. Specifically, suppose that the optimal solution $W^*(\lambda_{k})$ is known. Then, we apply DPC to identify the inactive features of MTFL model (\ref{prob:MTFL}) at $\lambda_{k+1}$ via $W^*(\lambda_{k})$. We repeat this process until all $W^*(\lambda_k)$, $k=1,\ldots,\mathcal{K}$ are computed. 
\begin{corollary}\label{corollary:DPCs}
	\textup{\textbf{DPC}} For the MTFL model \textup{(\ref{prob:MTFL})}, suppose that we are given a sequence of parameter values $\lambda_{\rm max}=\lambda_0>\lambda_1>\ldots>\lambda_{\mathcal{K}}$. Then, for any $k=1,2,\ldots,\mathcal{K}-1$, if $W^*(\lambda_k)$ is known, we have 
	\begin{align*}
		s_{\ell}(\lambda_{k+1},\lambda_k)<1\Rightarrow(\mathbf{w}^{\ell})^*(\lambda_{k+1})=0,
	\end{align*}	
	where $s_{\ell}(\lambda,\lambda_0)$ is given by Theorem \textup{\ref{thm:nonconvex}}.
\end{corollary}
We omit the proof of Corollary \ref{corollary:DPCs} since it is a direct application of Theorem \ref{thm:DPC}.

\section{Experiments}\label{section:experiments}

We evaluate DPC on both synthetic and real data sets. To measure the performance of DPC, we report the \emph{rejection ratio}, namely, the ratio of the number of inactive features identified by DPC to the actual number of inactive features. We also report the \emph{speedup}, i.e., the ratio of the running time of solver without screening to the running time of solver with DPC. The solver is from the SLEP package \citep{SLEP}. For each data set, we solve the MTFL model in (\ref{prob:MTFL}) along a sequence of $100$ tuning parameter values of $\lambda$ \mbox{equally} spaced on the logarithmic scale of ${\lambda}/{\lambda_{\rm max}}$ from $1.0$ to $0.01$. We only evaluate DPC since no existing screening rule is applicable for the MTFL model in (\ref{prob:MTFL}).

\subsection{Synthetic Studies}\label{subsection:experiments_synthetic}
We perform experiments on two synthetic data sets, called Synthetic 1 and Synthetic 2, that are commonly used in the literature \citep{Tibshirani2011,Zou2005}. Both synthetic 1 and Synthetic 2 have $50$ tasks. Each task contains $50$ samples. For $t=1,\ldots,50$, the true model is $$\mathbf{y}_t=\mathbf{X}_t\mathbf{w}_t^*+0.01\epsilon,\,\epsilon\sim N(0,1).$$ For Synthetic 1,  the entries of each data matrix ${\bf X}_t$ are i.i.d. standard Gaussian with
pairwise correlation zero, i.e., ${\rm corr}\left({\bf x}_i^{(t)},{\bf x}_j^{(t)}\right)=0$. For Synthetic 2, the entries of each data matrix ${\bf X}_t$ are drawn from i.i.d. standard Gaussian with pairwise correlation $0.5^{|i-j|}$, i.e., ${\rm corr}\left({\bf x}_i^{(t)},{\bf x}_j^{(t)}\right)=0.5^{|i-j|}$. To construct $\mathbf{w}_t^*$, we first randomly select $10\%$ of the features. Then, the corresponding components of $\mathbf{w}_t^*$ are populated from a standard Gaussian, and the remaining ones are set to $0$. For both Synthetic 1 and Synthetic 2, we set the feature dimension to $10000$, $20000$, and $50000$, \mbox{respectively}. For each setting, we run $20$ trials and report the average performance in \figref{fig:rej_ratio_synthetic} and Table \ref{table:DPC_runtime}.
\begin{figure*}[ht!]
	\centering{
		\subfigure[Synthetic 1, $d=10000$] { \label{fig:syn1_1}
			\includegraphics[width=0.3\columnwidth]{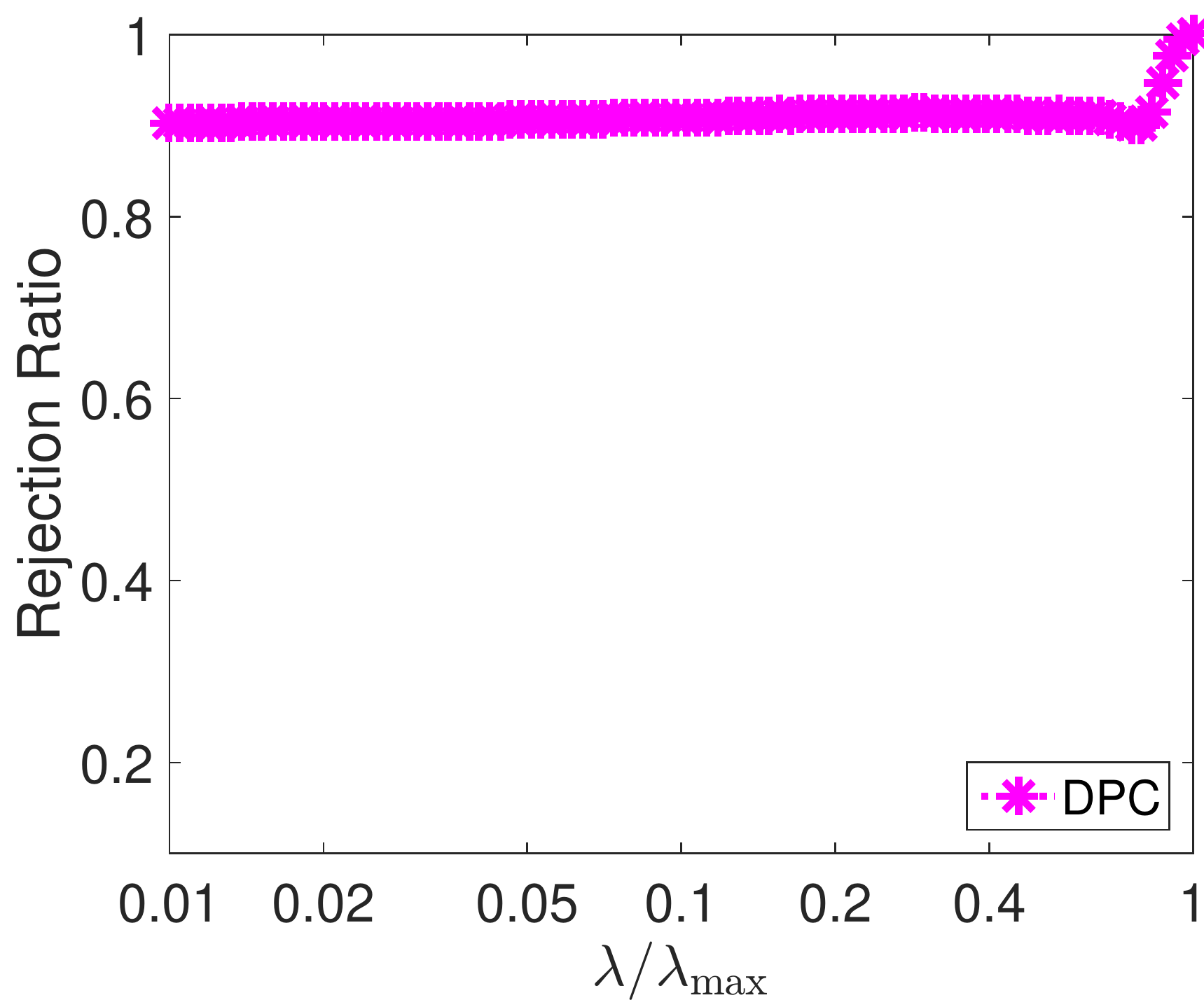}
		}
		\subfigure[Synthetic 1, $d=20000$] { \label{fig:syn1_2}
			\includegraphics[width=0.3\columnwidth]{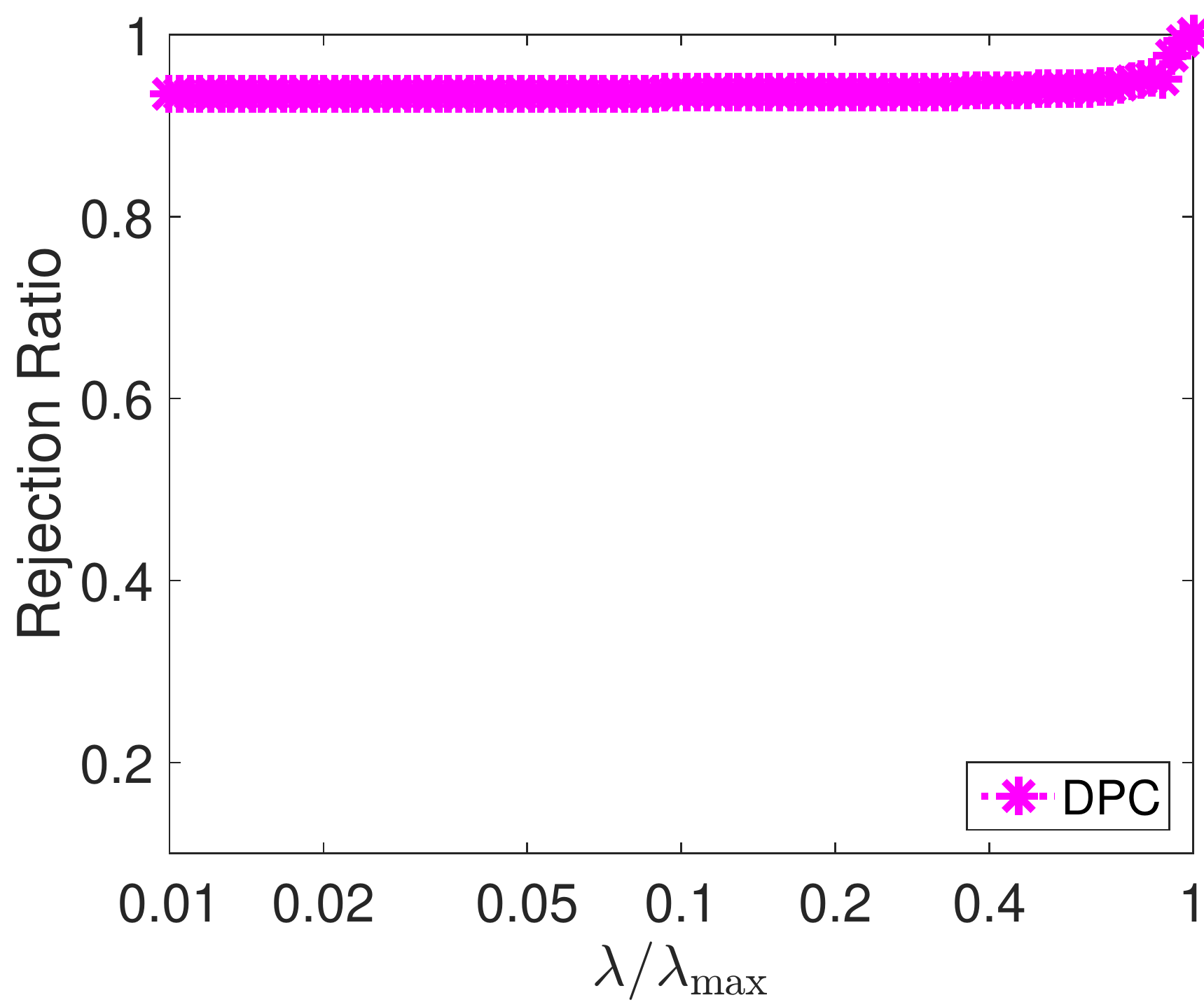}
		}
		\subfigure[Synthetic 1, $d=50000$] { \label{fig:syn1_3}
			\includegraphics[width=0.3\columnwidth]{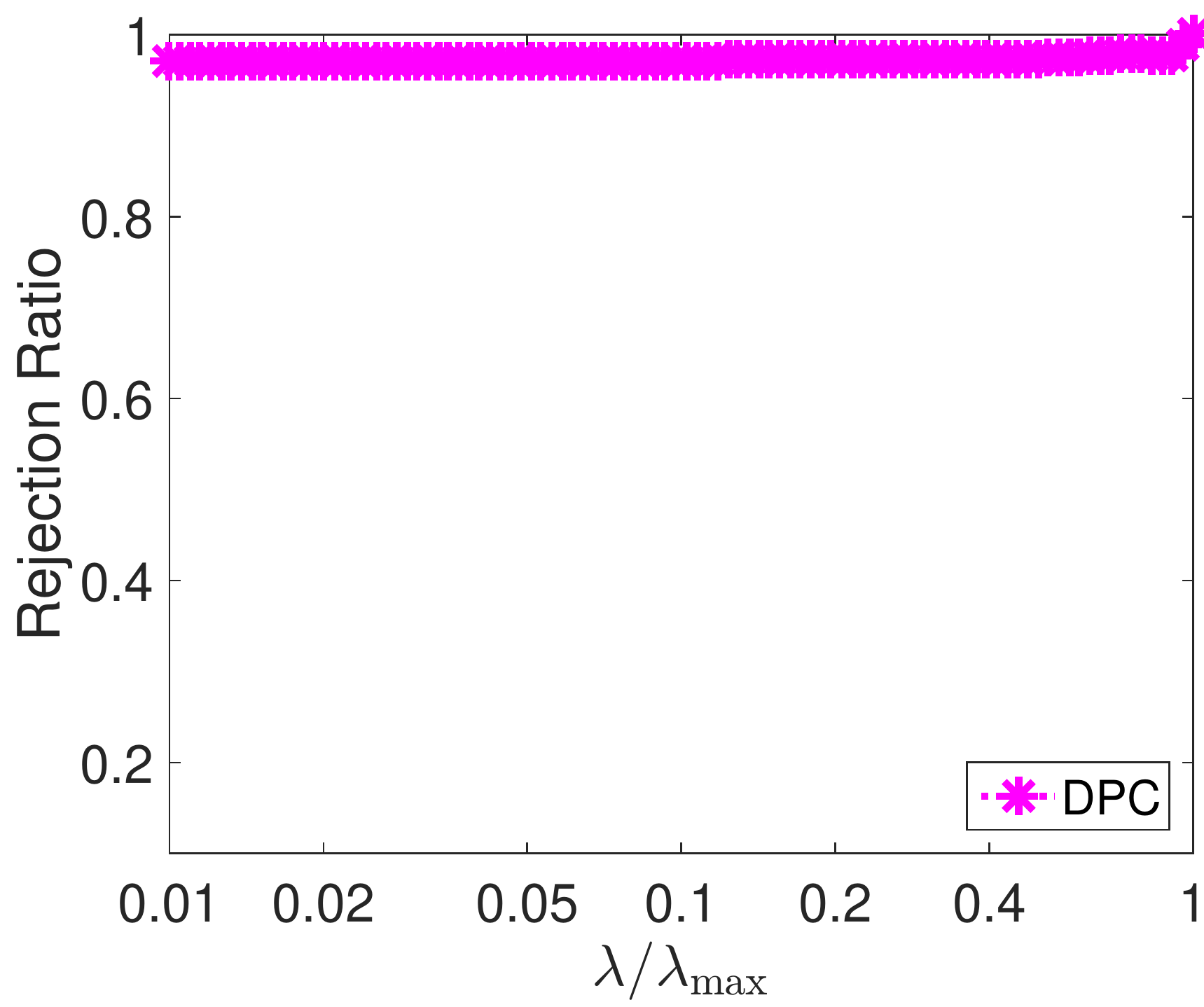}
		}
		\\
		\subfigure[Synthetic 2, $d=10000$] { \label{fig:syn2_1}
			\includegraphics[width=0.3\columnwidth]{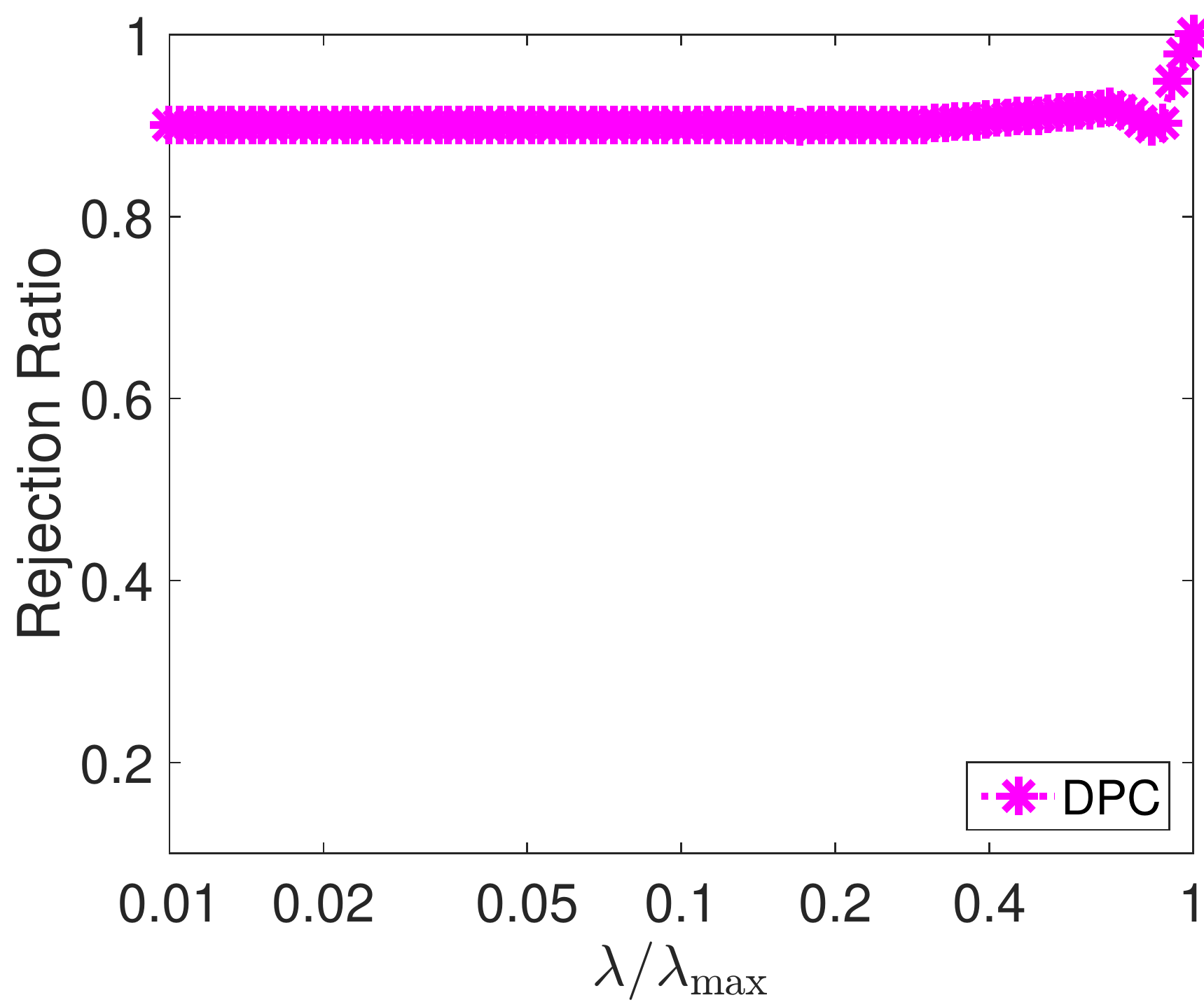}
		}
		\subfigure[Synthetic 2, $d=20000$] { \label{fig:syn2_2}
			\includegraphics[width=0.3\columnwidth]{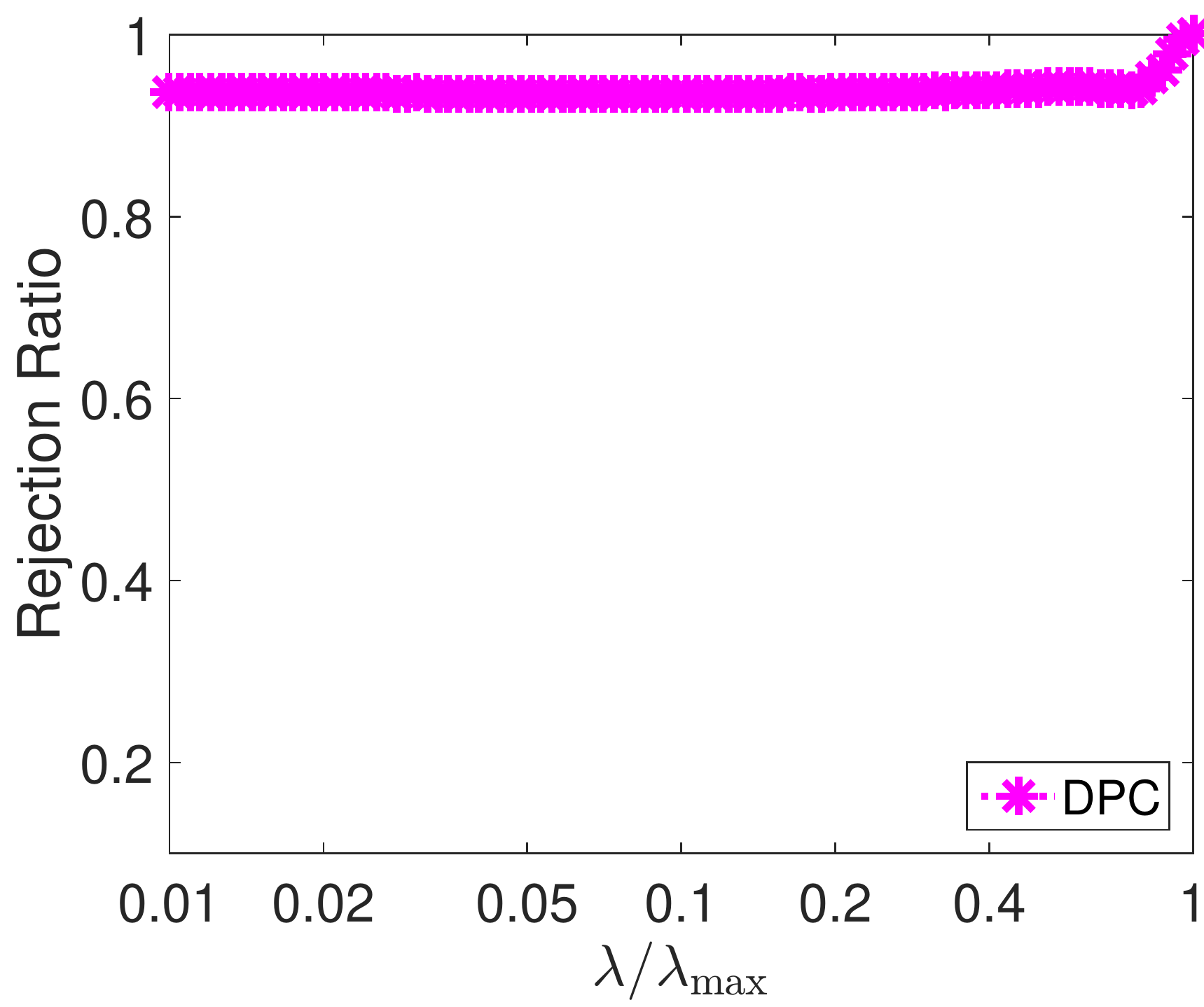}
		}
		\subfigure[Synthetic 2, $d=50000$] { \label{fig:syn2_3}
			\includegraphics[width=0.3\columnwidth]{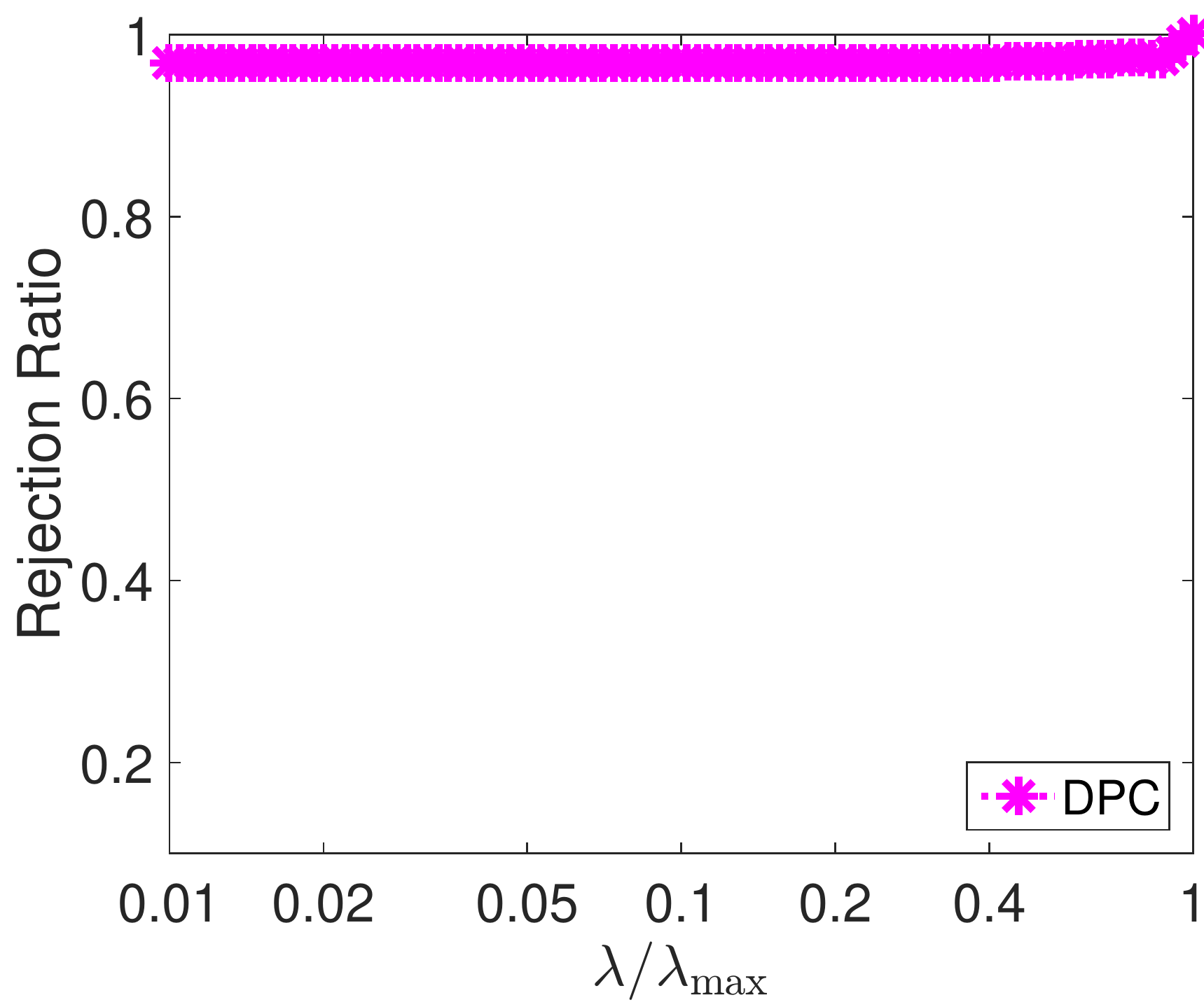}
		}
	}\vspace{-1mm}
	\caption{Rejection ratios of DPC on two synthetic data sets with different feature dimensions.}
	\label{fig:rej_ratio_synthetic}
	\vspace{-2mm}
\end{figure*}

\figref{fig:rej_ratio_synthetic} shows the rejection ratios of DPC on Synthetic 1 and Synthetic 2. For all the six settings, the rejection ratios of DPC are higher than $90\%$, even for small parameter values. This demonstrates one of the advantages of DPC, as previous empirical studies \citep{Ghaoui2012,Tibshirani2011,Wang-JMLR} indicate that the capability of screening rules in identifying inactive features usually decreases as the parameter value decreases. Moreover, \figref{fig:rej_ratio_synthetic} also shows that as the feature dimension increases, the rejection ratios of DPC become higher---that is very close to $1$. This implies that the potential capability of DPC in identifying the inactive features on high-dimensional data sets would be even more significant. 

Table \ref{table:DPC_runtime} presents the running time of the solver with and without \mbox{DPC}. The speedup is very significant, which is up to $60$ times. Take \mbox{Synthetic} 1 for example. When the feature dimension is $50000$, the solver without DPC takes about $40.68$ hours to solve problem (\ref{prob:MTFL}) at $100$ paramater values. In contrast, combined with \mbox{DPC}, the solver only takes less than one hour to solve the same $100$ problems---which leads to a speedup about $60$ times. Table \ref{table:DPC_runtime} also shows that the computational cost of DPC is very low---which is negligible compared to that of the solver without screening. Moreover, as the rejection ratios of DPC increases with feature dimension growth (see \figref{fig:rej_ratio_synthetic}), Table \ref{table:DPC_runtime} shows that the speedup by DPC increases as well.

\subsection{Experiments on Real Data Sets}

We perform experiments on three real data sets: 1) the \mbox{TDT2} text data set \citep{CWH09}; 2) the animal data set \citep{Lampert2009}; 3) the Alzheimer’s Disease Neuroimaging Initiative (ADNI) data set (\url{http://adni.loni.usc.edu/}).

\begin{figure*}[ht!]
	\centering{
		\subfigure[Animal, $d=15036$] { \label{fig:animal}
			\includegraphics[width=0.3\columnwidth]{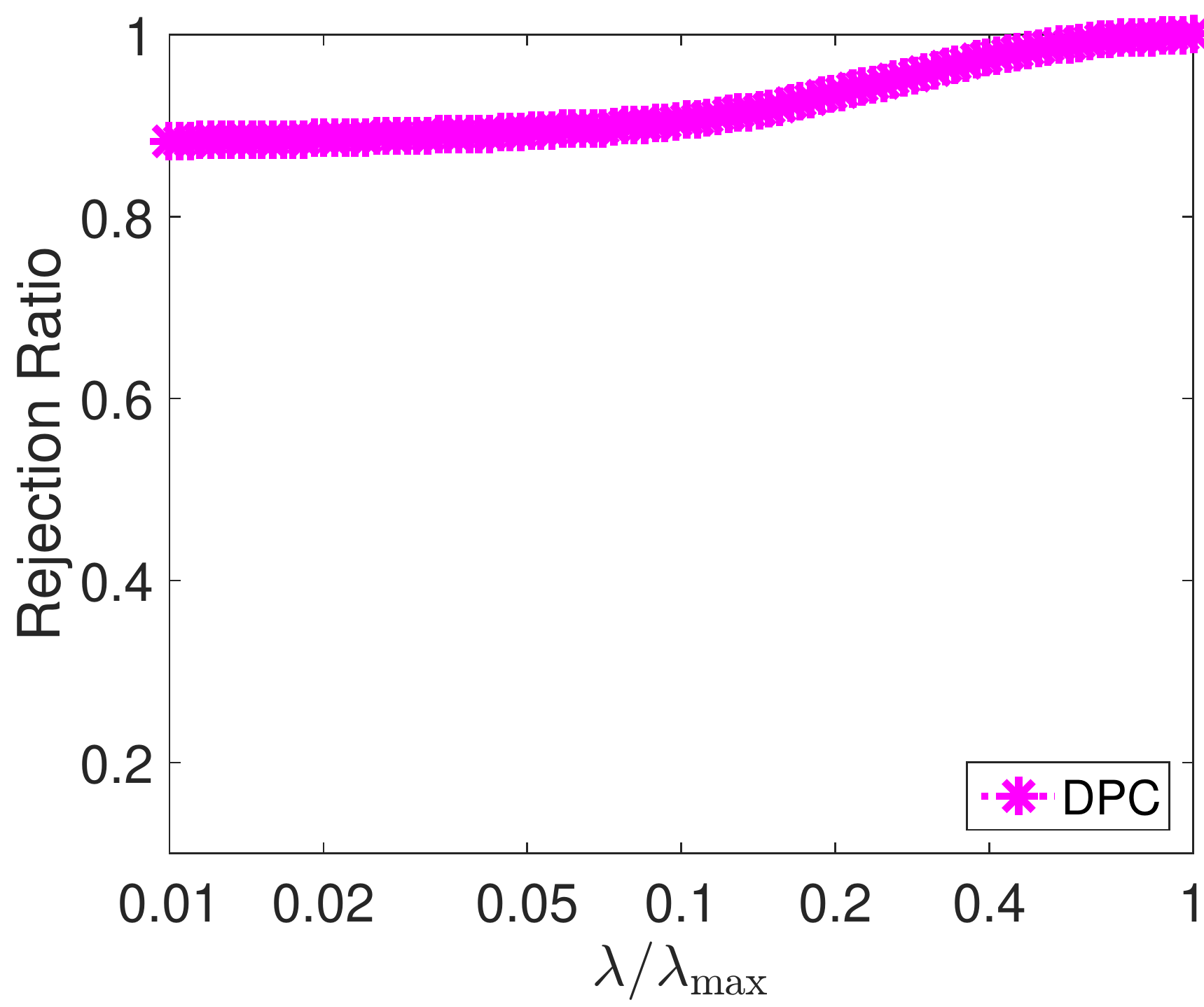}
		}
		\subfigure[TDT2, $d=24262$] { \label{fig:TDT}
			\includegraphics[width=0.3\columnwidth]{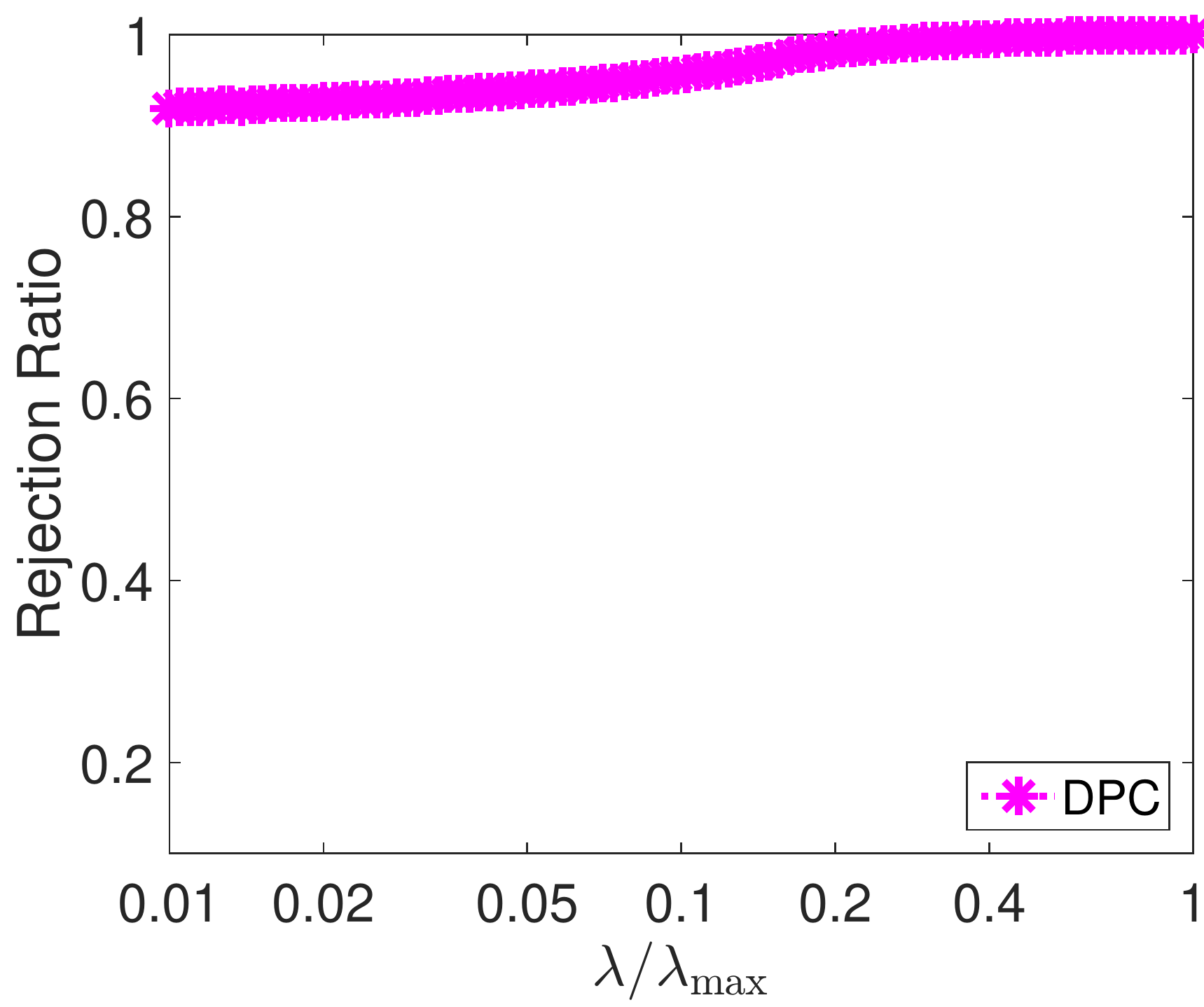}
		}
		\subfigure[ADNI, $d=504095$] { \label{fig:ADNI}
			\includegraphics[width=0.3\columnwidth]{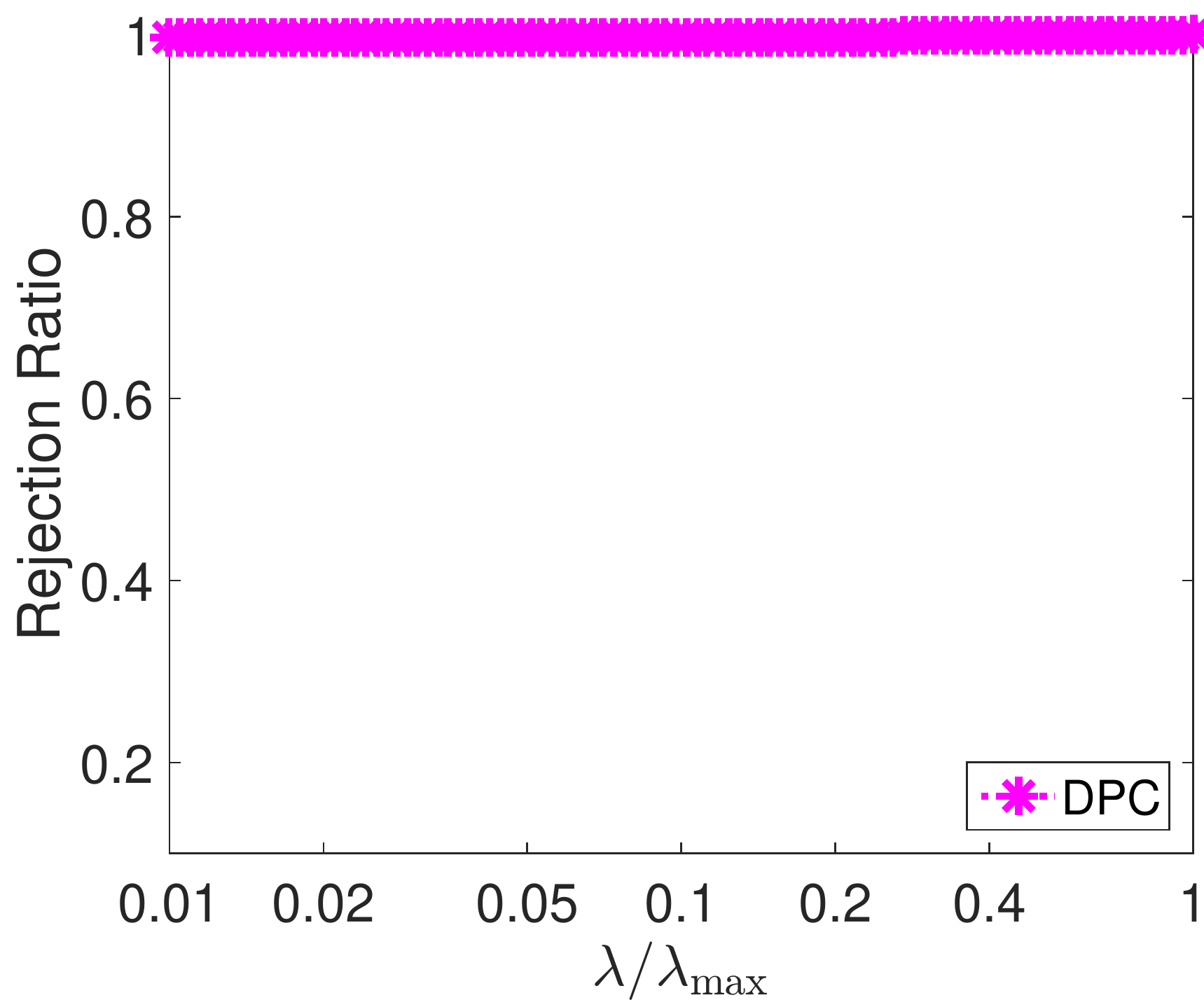}
		}
	}\vspace{-2mm}
	\caption{Rejection ratios of DPC on three real data sets.}
	\label{fig:rej_ratio_real}
	\vspace{-2mm}
\end{figure*}

\textbf{The Animal Data Set} The data set consists of $30475$ images of $50$ animals classes. By following the experiment settings in \citet{Kang2011}, we choose
$20$ animal classes in the data set: antelope,
grizzly-bear, killer-whale, beaver, Dalmatian, Persiancat,
horse, german- shepherd, blue-whale, Siamese-cat,
skunk, ox, tiger, hippopotamus, leopard, moose, spidermonkey,
humpback-whale, \mbox{elephant}, and gorilla. We construct $20$ tasks, where each of them is a classification task of one type of animal against all the others. For the $t^{th}$ task, we first \mbox{randomly} select $30$ samples from the $t^{th}$ class as the positive samples; and then we randomly select $30$ samples from all the other classes as the negative samples. We make use of all the seven sets of features kindly provided by \citet{Lampert2009}: color histogram features, local self-similarity features, PyramidHOG (PHOG) features, SIFT features, colorSIFT features, SURF features, and DECAF features. Thus, each image is represented by a $15036$-dimensional vectors. Hence, the data matrix $X_t$ of the $t^th$ task is of $60\times 15036$, where $t=1,\ldots,20$.

\textbf{The TDT2 Data Set} The original data set contains $9394$ documents of $30$ categories. Each document is represented by a $36771$-dimensional vector. Similar to the Animal data set, we construct $30$ tasks, each of which is a classification task of one category against all the others \citep{Amit2007}. Also, for the $t^{th}$ task, we first \mbox{randomly} select $50$ samples from the $t^{th}$ category as the positive samples, and then we randomly select $50$ samples from all the other categories as the negative samples. Moreover, we remove the features that have only zero entries, thus leaving us $24262$ features. Hence, the data matrix $X_t$ of the $t^th$ task is of $100\times 24262$, where $t=1,\ldots,30$. 

\setlength{\tabcolsep}{.18em}
\begin{table}[t]
	\vspace{-2mm}
	\begin{center}
		\caption{Running time (in minutes) for solving the MTFL model (\ref{prob:MTFL}) along a sequence of $100$ tuning parameter values of $\lambda$ \mbox{equally} spaced on the logarithmic scale of ${\lambda}/{\lambda_{\rm max}}$ from $1.0$ to $0.01$ by (a): the solver \citep{SLEP} without screening (see the third column); (b): the solver  with DPC (see the fifth column). 
		}\label{table:DPC_runtime}\vspace{2mm}
		\begin{footnotesize}
			\def\arraystretch{1.25}
			\begin{tabular}{ l c|c|c|c|c| }
				\cline{2-6}
				& \multicolumn{1}{|c|}{$d$} &  solver & DPC &  DPC+solver & \textbf{speedup} \\
				\cline{2-6}\\ [-2.5ex]\hline
				\multicolumn{1}{|r|}{\multirow{3}{*}{Synthetic 1}}  & \multicolumn{1}{|c|}{$10000$} & 405.75 & 0.7 & 28.12 & \textbf{14.43}  \\ \cline{2-6}
				\multicolumn{1}{|r|}{}  & \multicolumn{1}{|c|}{$20000$} & 913.70 & 1.36 & 37.02 & \textbf{24.68} \\\cline{2-6}
				\multicolumn{1}{|r|}{} & \multicolumn{1}{|c|}{$50000$} & 2441.57 & 3.50 & 42.08 & \textbf{58.03}  \\\hline\hline
				\multicolumn{1}{|r|}{\multirow{3}{*}{Synthetic 2}}  & \multicolumn{1}{|c|}{$10000$} & 406.85 & 0.70 & 29.28 & \textbf{13.89} \\\cline{2-6}
				\multicolumn{1}{|r|}{}  & \multicolumn{1}{|c|}{$20000$} & 906.09 & 1.37 & 36.66 & \textbf{24.72} \\\cline{2-6}
				\multicolumn{1}{|r|}{}  & \multicolumn{1}{|c|}{$50000$} & 2435.38 & 3.46 & 44.78 & \textbf{54.39}  \\\hline\hline
				\multicolumn{1}{|c|}{\multirow{1}{*}{Animal}} & 15036 & 311.71 & 0.47 & 16.36 & \textbf{19.05} \\\hline
				\multicolumn{1}{|c|}{\multirow{1}{*}{TDT2}} & 24262 & 958.66 & 1.87 & 44.11 & \textbf{21.74} \\\hline
				\multicolumn{1}{|c|}{\multirow{1}{*}{ADNI}} & 504095 & 9625.58 & 21.13 & 35.34 & \textbf{272.37} \\\hline
			\end{tabular}
		\end{footnotesize}
	\end{center}
	\vspace{-0.2in}	
\end{table}

\textbf{The ADNI Data Set} The data set consists of $747$ patients with $504095$ single nucleotide polymorphisms (SNPs), and the volume of $93$ brain regions for each patient. We first randomly select $20$ brain regions. Then, for each region, we randomly select $50$ patients, and utilize the corresponding SNPs data as the data matrix and the volumes of that brain region as the response. Thus, we have $20$ tasks, each of which is a regression task. The data matrix $X_t$ of the $t^{th}$ task is of $50\times 504095$, where $t=1,\ldots,20$. 

\figref{fig:rej_ratio_real} shows the rejection ratios of DPC---that are above $90\%$---on the aforementioned three real data sets. In particular, the rejection ratios of DPC on the ADNI data set are higher than $99\%$ at the $100$ parameter values. Table \ref{table:DPC_runtime} shows that the resulting speedup is very significant---that is up to $270$ times. We note that the feature dimension of the ADNI data set is more than \emph{half million}. Without screening, Table \ref{table:DPC_runtime}  shows that the solver takes about \emph{seven days} (approximately \emph{one week}) to compute the MTFL model (\ref{prob:MTFL}) at $100$ parameter values. However, integrated with the \mbox{DPC} screening rule, the solver computes the $100$ solutions in about \emph{half an hour}. The experiments again \mbox{indicate} that DPC provides better performance (in terms of rejection ratios and speedup) for higher dimensional data sets.
\vspace{2mm}
\section{Conclusion}\label{section:conclusion}

In this paper, we propose a novel screening method for the MTFL model in (\ref{prob:MTFL}), called DPC. The DPC screening rule is based on an indepth analysis of the geometric properties of the dual problem and the dual feasible set. To the best of our knowledge, DPC is the first screening rule that is applicable to sparse models with multiple data matrices. DPC is \emph{safe} in the sense that the identified features by DPC are guaranteed to have zero coefficients in the solution vectors across all tasks. Experiments on synthetic and real data sets demonstrate that DPC is very effective in identifying the inactive features, which leads to a substantial savings in computational cost and memory usage \emph{without sacrificing accuracy}. Moreover, DPC is more effective as the feature dimension increases, which makes DPC a very competitive candidate for the applications of very high-dimensional data. We plan to extend DPC to more general MTFL models, e.g., the MTFL models with multiple regularizers.

\clearpage
\newpage

\appendix

\section{Discussions regarding to the Dual Problem of (\ref{prob:MTFL})}

Although \eqref{eqn:MTFL_minf2_Fermat} implies that $\|\mathbf{m}^{\ell}\|\leq1$, this might not be the case. Thus, we need to consider the following two cases.
\begin{enumerate}
	\item If \eqref{eqn:MTFL_minf2_Fermat} holds, we can see that $\langle{\bf m}^{\ell},{\bf w}^{\ell}\rangle=\|{\bf w}^{\ell}\|$ and thus
	\begin{align}
	\min_{{\bf w}^{\ell}}\,f^{(\ell)}({\bf w}^{\ell})=0.
	\end{align}
	Therefore, we have
	\begin{align}
	\min_{W}\,f(W)=0.
	\end{align}
	\item If \eqref{eqn:MTFL_minf2_Fermat} does not hold, i.e., $\|{\bf m}^{\ell}\|>1$, we would have
	\begin{align}
	\inf_{{\bf w}^{\ell}}\,f^{(\ell)}({\bf w}^{\ell})=-\infty,
	\end{align}
	and thus
	\begin{align}
	\min_{W}\,f(W)=-\infty.
	\end{align}
	To see this, we define ${\bf w}^{(\ell)}(t)=t\frac{{\bf m}^{\ell}}{\|{\bf m}^{\ell}\|}$ and thus $$\langle{\bf m}^{\ell},{\bf w}^{(\ell)}(t)\rangle=t\|{\bf m}^{\ell}\|.$$ Then, we have
	\begin{align}
	f^{(\ell)}({\bf w}^{\ell}(t))=t(1-\|{\bf m}^{\ell}\|).
	\end{align}
	Because $\|{\bf m}^{\ell}\|>1$, the above equation yields
	\begin{align}
	\inf_{{\bf w}^{\ell}}\,f^{(\ell)}({\bf w}^{\ell})\leq\lim_{t\rightarrow\infty}\,f_2^{(\ell)}({\bf w}^{\ell}(t))=-\infty.
	\end{align}
\end{enumerate}
The above discussion implies that
\begin{align}\label{eqn:MTFL_minf2}
\min_{W}\,f(W)=
\begin{dcases}
0,\hspace{6mm}\mbox{if}\,\|{\bf m}^{\ell}\|\leq1,\,\ell=1,\ldots,d,\\
-\infty,\hspace{2mm}\mbox{otherwise}.
\end{dcases}
\end{align}

\section{Proof of Theorem \ref{thm:MTFL_primal_dual_closed_form}}

\begin{proof}
	For notational convenience, let
	\begin{enumerate}
		\item[\textup{1.}] $\dfrac{\bf y}{\lambda}\in\mathcal{F}$;
		\item[\textup{2.}] $\theta^*(\lambda)=\dfrac{\bf y}{\lambda}$;
		\item[\textup{3.}] $W^*(\lambda)=0$;
		\item[\textup{4.}] $\lambda\geq\lambda_{\rm max}$.
	\end{enumerate}
	
	\eqref{eqn:theta*_proj} implies that 1 is equivalent to 2.
	
	($2\Leftrightarrow3$) Suppose that 2 holds. \eqref{eqn:MTFL_KKT1} implies that $X_t{\bf w}^*_t(\lambda)=0$ for $t=1,\ldots, T$. Denote the objective function of the MTFL model (\ref{prob:MTFL}) by $f(W)$. We claim that $W^*(\lambda)$ must be zero. To see this, let $\overline{W}^*(\lambda)\neq0$ be another optimal solution of (\ref{prob:MTFL}) and thus $X_t\bar{\bf w}^*_t(\lambda)=0$ for $t=1,\ldots,T$. However, it is evident that $f(W^*(\lambda))<f(\overline{W}^*(\lambda))$. This leads to a contradiction. Thus, the optimal solution $W^*(\lambda)$ is zero and we have proved $2\Rightarrow3$. The converse direction, i.e., $2\Leftarrow3$ is a direct consequence of \eqref{eqn:MTFL_KKT1}.
	
	($1\Leftrightarrow4$) It is evident that 1 holds if and only if ${\bf y}/\lambda$ is a feasible solution of problem (\ref{prob:MTFL_dual}), namely, all constraints in (\ref{prob:MTFL_dual}) holds at ${\bf y}/\lambda$. By plugging ${\bf y}/\lambda$ into the constraints in (\ref{prob:MTFL_dual}), we can see that the feasibility of ${\bf y}/\lambda$ is equivalent to 4. Thus, we can see that 1 is equivalent to 4. This completes the proof.
\end{proof}

\section{Proof of Corollary \ref{corollary:convex_set_projection}}

\begin{proof}
	\begin{enumerate}
		\item[\textup{1.}] To show part 1, we only need to set ${\bf u}_1={\bf u}$ and ${\bf u}_2=0$, and then plug them into the inequality (\ref{ineqn:nonexpansive}) [note that $\textup{P}_{\mathcal{C}}(0)=0$ since $0\in\mathcal{C}$]. 
		\item[\textup{2.}] Part 1 implies that $\|\textup{P}_{\mathcal{C}}({\bf u})\|\leq\|{\bf u}\|$. Thus, we have
		\begin{align*}
		\|{\bf u}\|^2\geq\|{\bf u}\| \|\textup{P}_{\mathcal{C}}({\bf u})\|\geq\langle{\bf u},\textup{P}_{\mathcal{C}}({\bf u})\rangle,
		\end{align*}
		which is equivalent to the statement in part 2.
	\end{enumerate}
	The proof is completed.
\end{proof}

\section{Proof of Theorem \ref{thm:MTFL_estimation}}

We first cite some useful properties of the projection operators.

\begin{lemma}\label{lemma:normal_cone_projection}
	\textup{\citep{Ruszczynski2006,Bauschke2011}} Let $\mathcal{C}$ be a nonempty closed convex set of a Hilbert space and $\mathbf{u}\in\mathcal{C}$. Then
	\begin{enumerate}
		\item[\textup{1.}]
		$N_{\mathcal{C}}(\mathbf{u})=\{\mathbf{v}:\mathbf{P}_{\mathcal{C}}(\mathbf{u}+\mathbf{v})=\mathbf{u}\}$.
		\item[\textup{2.}]  $\mathbf{P}_{\mathcal{C}}(\mathbf{u}+\mathbf{v})=\mathbf{u},\,\forall\,\mathbf{v}\in N_{\mathcal{C}}(\mathbf{u})$.
		\item[\textup{3.}]  Let $\overline{\mathbf{u}}\notin\mathcal{C}$ and $\mathbf{u}=\mathbf{P}_{\mathcal{C}}(\overline{\mathbf{u}})$. Then, $\mathbf{P}_{\mathcal{C}}(\mathbf{u}+t(\overline{\mathbf{u}}-\mathbf{u}))=\mathbf{u}$ for all $t\geq0$.
	\end{enumerate}
\end{lemma}

We are now ready to prove Theorem \ref{thm:MTFL_estimation}

\begin{proof}~
	\begin{enumerate}
		\item For $\lambda\in(0,\lambda_{\rm max})$, Theorem \ref{thm:MTFL_primal_dual_closed_form} implies that ${\bf y}/\lambda\notin\mathcal{F}$. Thus, the statement holds for $\lambda\in(0,\lambda_{\rm max})$ by Theorem \ref{thm:normal_cone} and \eqref{eqn:theta*_proj} [let $\bar{\bf u}={\bf y}/\lambda$ and ${\bf u}=\theta^*(\lambda)$]. 
		
		To show the statement holds at $\lambda_{\rm max}$, Theorem \ref{thm:normal_cone} indicates that we need to show
		\begin{align}\label{ineqn:gradient_g*_normal_cone}
		\left\langle\nabla g_{\ell_*}\left(\frac{\bf y}{\lambda_{\rm max}}\right),\theta-\frac{\bf y}{\lambda_{\rm max}}\right\rangle\leq0,\,\forall \theta\in\mathcal{F}.
		\end{align}
		Because $g_{\ell_*}(\cdot)$ is convex, we have \citep{Ruszczynski2006}
		\begin{align}\label{ineqn:g*_convex}
		\hspace{-3mm}g_{\ell_*}(\theta)-g_{\ell_*}\left(\frac{{\bf y}}{\lambda_{\rm max}}\right)\geq\left\langle\nabla g_{\ell_*}\left(\frac{\bf y}{\lambda_{\rm max}}\right),\theta-\frac{\bf y}{\lambda_{\rm max}}\right\rangle.
		\end{align}
		Note that, $g_{\ell}$ is the constraint function of the dual problem in (\ref{prob:MTFL_dual}). Thus, for any dual feasible solution $\theta\in\mathcal{F}$, it is evident that $g_{\ell_*}(\theta)\leq1$. Moreover, \eqref{eqn:MTFL_lambdamx} implies that $g_{\ell_*}({\bf y}/\lambda_{\rm max})=1$. Therefore, the left hand of the inequality (\ref{ineqn:g*_convex}) must be non-positive, which yields inequality (\ref{ineqn:gradient_g*_normal_cone}). Thus, the statement holds.
		\item A direct application of part 2 of Corollary \ref{corollary:convex_set_projection} yields 
		\begin{align*}
		\left\langle\frac{\mathbf{y}}{\lambda_0},\mathbf{n}(\lambda_0)\right\rangle=\left\langle\frac{\mathbf{y}}{\lambda_0},\frac{\mathbf{y}}{\lambda_0}-\theta^*(\lambda_0)\right\rangle\geq0,\,\forall\,\lambda_0\in(0,\lambda_{\rm max}).
		\end{align*}
		When $\lambda_0=\lambda_{\rm max}$, by noting that $\mathbf{n}(\lambda_{\rm max})=\nabla g_{\ell_*}(\tfrac{\mathbf{y}}{\lambda_{\rm max}})$, we have
		\begin{align*}
		\left\langle\frac{\mathbf{y}}{\lambda_{\rm max}},\mathbf{n}(\lambda_{\rm max})\right\rangle=\sum_{t=1}^T2\left\langle\mathbf{x}_{\ell_*}^{(t)},\frac{\mathbf{y}}{\lambda_{\rm max}}\right\rangle^2\geq0.
		\end{align*}
		Thus, the statement holds.
		\item By \eqref{eqn:MTFL_r}, we have
		\begin{align}\label{eqn:ip_rn}
		\langle\textbf{\textup{r}}(\lambda,\lambda_0),\textbf{\textup{n}}(\lambda_0)\rangle=\left(\frac{1}{\lambda}-\frac{1}{\lambda_0}\right)\langle{\bf y},\textbf{\textup{n}}(\lambda_0)\rangle+\left\langle\frac{\bf y}{\lambda_0}-\theta^*(\lambda_0),\textbf{\textup{n}}(\lambda_0)\right\rangle.
		\end{align}
		By \eqsref{eqn:MTFL_n} and (\ref{eqn:theta*_proj}), the second term on the right hand side of \eqref{eqn:ip_rn} is nonnegative for all $\lambda_0\in(0,\lambda_{\rm max}]$.
		
		The fact that $0\in\mathcal{F}$ yields 
		\begin{align*}
		\left\langle0-\frac{\bf y}{\lambda_{\rm max}},\textbf{\textup{n}}(\lambda_{\rm max})\right\rangle\leq0.
		\end{align*} 
		Thus, the first term on the right hand side of \eqref{eqn:ip_rn} is nonnegative for $\lambda_0=\lambda_{\rm max}$. For $\lambda_0\in(0,\lambda_{\rm max})$, part 2 of Corollary \ref{corollary:convex_set_projection}, \eqsref{eqn:theta*_proj} and (\ref{eqn:MTFL_n}) imply that
		\begin{align*}
		\left\langle\frac{\bf y}{\lambda_0},\frac{\bf y}{\lambda_0}-\textup{P}_{\mathcal{F}}\left(\frac{\bf y}{\lambda_0}\right)\right\rangle=\left\langle\frac{\bf y}{\lambda_0},\textbf{\textup{n}}(\lambda_0)\right\rangle\geq0.
		\end{align*}
		Thus, the first term on the right hand side of \eqref{eqn:ip_rn} is nonnegative for $\lambda_0\in(0,\lambda_{\rm max})$. 
		
		As a result, the inner product $\langle\textbf{\textup{r}}(\lambda,\lambda_0),\textbf{\textup{n}}(\lambda_0)\rangle$ is nonnegative.
		\item We define
		\begin{align}
		\theta(t)=\theta^*(\lambda_0)+t\textbf{\textup{n}}(\lambda_0).
		\end{align}
		Part 1 of Lemma \ref{lemma:normal_cone_projection} implies that
		\begin{align}\label{eqn:projection_ray}
		\textup{P}_{\mathcal{F}}(\theta(t))=\theta^*(\lambda_0),\,\forall\,t\geq0.
		\end{align} 
		The nonexpansiveness of the projection operators yields [let ${\bf u}_1={\bf y}/\lambda$ and ${\bf u}_2=\theta(t)$ and plug them into (\ref{ineqn:nonexpansive})]
		\begin{align*}
		\left\|\textup{P}_{\mathcal{F}}\left(\frac{\bf y}{\lambda}\right)-\textup{P}_{\mathcal{F}}(\theta(t))\right\|^2+\|(\textup{P}_{\mathcal{F}}-\textup{Id})\left(\frac{\bf y}{\lambda}\right)-(\textup{P}_{\mathcal{F}}-\textup{Id})(\theta(t))\|^2\leq\left\|\frac{\bf y}{\lambda}-\theta(t)\right\|^2,\,\forall\,t\geq0.
		\end{align*}
		By \eqsref{eqn:theta*_proj}, (\ref{eqn:projection_ray}) and (\ref{eqn:MTFL_r}), the above inequality reduces to
		\begin{align}\label{ineqn:MTFL_estimation1}
		\|\theta^*(\lambda)-\theta^*(\lambda_0)\|^2+\|\theta^*(\lambda)-\theta^*(\lambda_0)-(\textbf{\textup{r}}(\lambda,\lambda_0)-t\textbf{\textup{n}}(\lambda_0))\|^2\leq\|\textbf{\textup{r}}(\lambda,\lambda_0)-t\textbf{\textup{n}}(\lambda_0)\|^2,\,\forall\,t\geq0.
		\end{align}
		Let us consider
		\begin{align}\label{prob:radius_ball}
		\min_{t\geq0}\,r(t)=\|\textbf{\textup{r}}(\lambda,\lambda_0)-t\textbf{\textup{n}}(\lambda_0)\|^2.
		\end{align}
		Because $r(t)$ is a quadratic function of $t$, we can see that
		\begin{align*}
		\min_{t\geq0}\,r(t)=
		\begin{dcases}
		\|\textbf{\textup{r}}(\lambda,\lambda_0)\|^2,\hspace{4mm}\mbox{if}\,\langle\textbf{\textup{r}}(\lambda,\lambda_0),\textbf{\textup{n}}(\lambda_0)\rangle<0,\\
		\|\textbf{\textup{r}}^{\perp}(\lambda,\lambda_0)\|^2,\hspace{2mm}\mbox{if}\,\langle\textbf{\textup{r}}(\lambda,\lambda_0),\textbf{\textup{n}}(\lambda_0)\rangle\geq0.
		\end{dcases}
		\end{align*}

		Because of part 3, we have
		\begin{align}\label{eqn:MTFL_min_rt}
		\min_{t\geq0}\,r(t)&=\|\textbf{\textup{r}}^{\perp}(\lambda,\lambda_0)\|^2\\\label{eqn:MTFL_argmin_rt}
		\argmin_{t\geq0}\,r(t)&=\frac{\langle\textbf{\textup{r}}(\lambda,\lambda_0),\textbf{\textup{n}}(\lambda_0)\rangle}{\|\textbf{\textup{n}}(\lambda_0)\|^2}
		\end{align}
		Plugging \eqsref{eqn:MTFL_min_rt} and (\ref{eqn:MTFL_argmin_rt}) into (\ref{ineqn:MTFL_estimation1}) yields the statement, which completes the proof.
	\end{enumerate}
	The proof is complete.
\end{proof}

\clearpage
\newpage
{
	\bibliographystyle{plainnat}
	\bibliography{refs}
}

\end{document}